\documentclass[twoside,11pt]{article}

\makeatletter
\def\ps@plain{%
  \let\@oddhead\@empty
  \let\@evenhead\@empty
  \def\@oddfoot{\reset@font\hfil\thepage\hfil}%
  \def\@evenfoot{\reset@font\hfil\thepage\hfil}%
}
\makeatother

\usepackage{jmlrmod}

\usepackage{amssymb,thmtools,mathtools,bm}
\usepackage{array}
\usepackage{graphicx}
\usepackage{algorithm}
\usepackage{algpseudocode}
\usepackage{physics}
\usepackage{tikzsymbols}
\usepackage{array,tabularx,multirow}
\usepackage{makecell}
\newcolumntype{P}[1]{>{\centering\arraybackslash}p{#1}}

\usepackage{caption}
\usepackage{subcaption}
\usepackage{verbatim}
\usepackage{enumerate}
\usepackage{parskip}

\usepackage[many]{tcolorbox}
\usetikzlibrary{decorations}

\usepackage{color,soul}
\definecolor{clemson-orange}{RGB}{234,106,32}
\definecolor{highlight-orange}{RGB}{255,150,150}
\definecolor{chicago-maroon}{RGB}{128,0,0}
\definecolor{cincinnati-red}{RGB}{190,0,0}
\definecolor{soft-cyan}{RGB}{68,85,90}
\definecolor{firebrick}{RGB}{178,34,34}
\definecolor{crimson}{RGB}{220,20,60}
\definecolor{cerrulean}{rgb}{0.165,0.322,0.745}
\definecolor{jaam}{rgb}{0.45,0.0,0.45}

\usepackage{hyperref}
\hypersetup{
  colorlinks   = true,    
  urlcolor     = jaam,    
  linkcolor    = firebrick,    
  citecolor    = blue      
}

\usepackage{parskip}

\newif\ifsolutions \solutionstrue

\def\final{0}
\newcommand{\reviewer}[3]{
  \expandafter\newcommand\csname #1\endcsname[1]{
    \ifthenelse{\equal{\final}{1}} {
      \textcolor{#3}{}
    } {
      \textcolor{#3}{\begin{center} \textbf{#2} ##1 \end{center}}
    }
  }
}

\reviewer{anirbit}{}{jaam}

\renewcommand{\ip}[2]{\left\langle#1,#2\right\rangle}

\renewcommand{\Pr}{\mathbb{P}}

\def\1{\bm{1}}

\newcommand{\E}{\mathbb{E}}

\def\va{{\bm{a}}}

\def\vu{{\bm{u}}}

\def\vw{{\bm{w}}}
\def\vx{{\bm{x}}}
\def\vy{{\bm{y}}}
\def\vz{{\bm{z}}}

\def\mE{{\bm{E}}}

\def\mW{{\bm{W}}}

\DeclareMathAlphabet{\mathsfit}{\encodingdefault}{\sfdefault}{m}{sl}
\SetMathAlphabet{\mathsfit}{bold}{\encodingdefault}{\sfdefault}{bx}{n}

\def\gB{{\mathcal{B}}}

\def\gD{{\mathcal{D}}}

\def\gH{{\mathcal{H}}}

\def\gL{{\mathcal{L}}}
\def\gM{{\mathcal{M}}}

\def\gP{{\mathcal{P}}}

\def\gU{{\mathcal{U}}}

\newcommand{\bb}{\mathbb}

\newcommand{\R}{\bb R}

\usepackage[capitalise,noabbrev]{cleveref}

\usepackage{hyperref,bbding}
\usepackage{url}

\ShortHeadings{Noisy PDE Training Requires Bigger PINNs}{Andre-Sloan and Mukherjee and Colbrook}
\firstpageno{1}

\newtheorem{theorem}{Theorem}[section]
\newtheorem{lemma}[theorem]{Lemma}

\newtheorem{definition}[theorem]{Definition}

\crefname{assumption}{Assumption}{Assumptions}
\newtheorem{remark}[theorem]{Remark}

\begin{document}

\title{Noisy PDE Training Requires Bigger PINNs}

\author{\name Sebastien Andre-Sloan \email sebastien.andre-sloan@postgrad.manchester.ac.uk \\
      \addr Department of Computer Science\\
      The University of Manchester
      \AND
\name Anirbit Mukherjee\thanks{corresponding author} \email anirbit.mukherjee@manchester.ac.uk \\
      \addr Department of Computer Science\\
      The University of Manchester
      \AND
      \name Matthew Colbrook \email mjc249@cam.ac.uk \\
      \addr Department of Applied Mathematics and Theoretical Physics\\
      The University of Cambridge
      }


\maketitle
\thispagestyle{plain}  

\begin{abstract}%

Physics-Informed Neural Networks (PINNs) are increasingly used to approximate solutions of partial differential equations (PDEs), particularly in high dimensions. In real-world settings, data are often noisy, making it crucial to understand when a predictor can still achieve low empirical risk. Yet, little is known about the conditions under which a PINN can do so effectively. We analyse PINNs applied to the Hamilton--Jacobi--Bellman (HJB) PDE and establish a lower bound on the network size required for the supervised PINN empirical risk to fall below the variance of noisy supervision labels. Specifically, if a predictor achieves empirical risk $\mathcal{O}(\eta)$ below $\sigma^2$ (the variance of the supervision data), then necessarily $d_N\log d_N\gtrsim N_s \eta^2$, where $N_s$ is the number of samples and $d_N$ the number of trainable parameters. A similar constraint holds in the fully unsupervised PINN setting when boundary labels are noisy. Thus, simply increasing the number of noisy supervision labels does not offer a ``free lunch'' in reducing empirical risk. We also give empirical studies\footnote{The codes are made available at this \href{https://github.com/SebastienAndreSloan/PINN\_Lowerbound.git}{(link)}.}  ~on the HJB PDE, the Poisson PDE and the the Navier-Stokes PDE set to produce the Taylor-Green solutions. In these experiments we demonstrate that PINNs indeed need to be beyond a threshold model size for them to train to errors below $\sigma^2$. These results provide a quantitative foundation for understanding parameter requirements when training PINNs in the presence of noisy data.


\end{abstract}


\clearpage 

\section{Introduction}

Partial differential equations (PDEs) underpin nearly all natural sciences and engineering, yet solving them remains a central challenge. A key difficulty lies in approximating their solutions without falling victim to the curse of dimensionality \citep{E_2021, han2018solving}, which affects even highly structured systems \citep{Beck_2021, berner2020numerically}. Deep learning has recently re-emerged as a powerful tool for this task, building on earlier efforts \citep{Lagaris_1998}. The broader trend of applying AI to complex scientific problems has given rise to the field of ``AI for science,'' encompassing diverse neural approaches to PDEs \citep{kaiser2021datadriven, erichson2019physicsinformed, Wandel_2021, li2022learning, salvi2022neural}. Recent theoretical advances in operator learning further contextualise our work. For instance, \citet{nicolaselliptic} show that certain elliptic PDEs can be learned from remarkably few data, providing an algorithm with exponentially decaying error in the number of input–output samples. A broader overview is given in the survey by \citet{BOULLE202483}.

Among neural PDE solvers, Physics-Informed Neural Networks (PINNs) \citep{raissi2019physics} have become a popular framework, effective across a wide range of problems, including the Navier–Stokes equations \citep{ARTHURS2021110364, wang2020physicsinformed, Eivazi_2022}. PINNs approximate PDE solutions with neural networks trained via a loss function that penalises violations of the differential operator, as well as initial and boundary conditions. Because these terms are defined over distinct domains, the resulting optimisation involves a multi-distribution loss, distinguishing PINNs from standard machine-learning models.

Theoretical understanding of their generalisation properties remains limited \citep{kutyniok2020theoretical, guhring2019error, geist2020numerical}. Recent progress \citep{Mishra_PINN_generalization_error} has yielded error bounds applicable to general PDEs, though these results rely on quadrature-based formulations differing from the standard PINN setup. Despite extensive study \citep{DeRyck_Mishra_2024}, it remains unclear why minimising the standard PINN loss often produces accurate PDE approximations in practice.

The surge of interest in PINNs has also revealed their difficulties with even simple PDEs \citep{krishnapriyan2021characterizing, wang2022pinns}. To address these limitations, researchers have explored alternative optimisation methods \citep{wang2020understanding} and hybrid strategies such as stitching together local solutions \citep{hu2022XPINNs}. Other approaches aim to alleviate the cost of computing derivatives in high-dimensional settings by replacing them with stochastic estimators \citep{hu2023exactsolutions, he2023learning, hu2024hutchtrace}.

A further challenge in understanding PINNs lies in quantifying errors and uncertainties, particularly when training on noisy data \citep{karniadakisnoisyuq}. Several variants, including the Uncertain PINN \citep{ZHANGUPINN} and Bayesian PINN \citep{YANGBPINN}, have been developed for this purpose. These issues are especially relevant in applications such as image super-resolution for medical diagnostics \citep{frangi2023superresolution}.

This paper establishes a lower bound on the size of PINNs required to achieve small losses. Our results uncover a fundamental link between model complexity and the amount of data needed for accurate performance. In particular, we show that maintaining low loss under noisy supervision requires the model size to scale with the data. This challenges the notion that noise can improve learning “for free”: whether the noise arises from initial conditions or solution samples, the network must exceed a critical size for it to be beneficial. These findings have practical implications for designing efficient models in realistic, noise-prone settings.

\subsection{The Setup} Consider a PDE of the form
$$\setlength\abovedisplayskip{6pt}\setlength\belowdisplayskip{6pt}
\gL(u)=f, \quad\gB(u)= u_B,
$$
where $\gL$ is a differential operator that acts on functions on $\tilde{D} \times [0,T]$, for a domain $\tilde{D}\subset\mathbb{R}^d$, and $\gB$ is a boundary (or trace) operator that defines an evaluation of functions on the boundary of the domain. The goal of the PINN method is to find an approximate solution $u_\vw$ (paramtrized by weights $\vw$) such that $\gL(u_\vw)\approx f$ and $\gB(u_\vw)\approx u_B$. To do this, one typically uses three kinds of ``residual'' functions. The first is the PDE residual:
$$\setlength\abovedisplayskip{6pt}\setlength\belowdisplayskip{6pt}
R_{\mathrm{PDE}}[u_\vw](\vx,t) = \gL(u_\vw)(\vx,t)-f(\vx,t),\quad (\vx,t)\in \tilde{D}\times[0,T].
$$
For time-dependent PDEs with initial condition $u(\vx,0)=g(\vx)$, there is also a temporal initial value residual:
$$\setlength\abovedisplayskip{6pt}\setlength\belowdisplayskip{6pt}
R_{\mathrm T}[u_\vw](\vx) = u_\vw(\vx,0)-g(\vx),\quad \vx\in \tilde{D}.
$$
In the semi-supervised setting, there is also a loss term based on data sets, corresponding to an ``observable'' $\Psi$ in a subdomain $\tilde{D'}\times[0,T]$, on which (possibly noisy) observations $\tilde{g}$ are available of a functional $\Psi$ of the true solution. In real-world applications, it is rare to have access to noiseless data, so we include noisy in our supervised labels to capture the difficulty of observing data. This leads to the data or supervised residual:
$$\setlength\abovedisplayskip{6pt}\setlength\belowdisplayskip{6pt}
R_{\mathrm S}[u_\vw](\vx,t) = \Psi[u_\vw](\vx,t)-\tilde{g}(\vx,t),\quad (\vx,t)\in \tilde{D'} \times[0,T].
$$
In total, the PINN risk function is given as:
\begin{align}\label{eq:genpinn}\setlength\abovedisplayskip{6pt}\setlength\belowdisplayskip{6pt}
    R_{\mathrm{PINN}}(u_\vw) = \int_{\tilde{D}\times[0,T]}R_{\mathrm{PDE}}[u_\vw]^2\dd{\vx} \dd{t} + \lambda_0\int_{\tilde{D}}R_{\mathrm T}[u_\vw]^2\dd{\vx} + \lambda_s\int_{\tilde{D'}\times[0,T]}R_{\mathrm S}[u_\vw]^2\dd{\vx} \dd{t},
\end{align}
where $\lambda_0,\lambda_s>0$ are parameters used to weight the different residuals and losses. A natural choice is $\Psi=\gB,\tilde{g}=u_B$ and $\gD' = \partial \gD$ (the boundary of $\gD$), where the above penalizes $u_\vw$ not satisfying the boundary condition. Given a parameter space $\Theta$, the corresponding risk minimization/learning problem is to solve, $\min_{\vw\in\Theta} R_{\mathrm{PINN}}(u_\vw)$.

In this paper, we consider a fundamental open question:

\noindent\textbf{Question:}
{\em Given noisy observations with variance $\sigma^2$ at random locations within a PDE domain, under what conditions can a predictor achieve a supervised empirical PINN risk below $\sigma^2$?}

Answering this question is key to understanding whether PINNs can use noisy data to solve PDEs accurately. In this work, we focus on solving the Hamilton--Jacobi--Bellman (HJB) PDE using supervised PINNs, and we derive the first necessary condition on model size and supervision for this desirable outcome. The HJB PDE is a nonlinear equation that governs the time evolution of a scalar function. It arises naturally in optimal stochastic control problems (see \cref{sec:control}) and is widely used in robotics research \citep{caoimpact,abu2005nearly,lewis2012optimal}. Its nonlinearity, high dimensionality, and importance make it an ideal test case for understanding the requirements PINNs must meet to successfully tackle complex PDEs under noisy conditions.

\subsection{Informal Statement of the Main Result}

Our main result (\cref{thm:main}) can be summarised as follows:

{\em 
Consider the PINN loss function for the HJB PDE. Let $\sigma^2$ be the average variance in the supervision data. If a predictor achieves an empirical risk $\mathcal{O}(\eta)$ below $\sigma^2$, then necessarily $d_N\log d_N\gtrsim N_s \eta^2$, where $N_s$ is the number of samples and $d_N$ is the number of trainable parameters of the PINN.}

This result shows that it is not necessarily useful to add more samples to the supervised loss term once a certain threshold is reached, but this threshold value may be increased by increasing the number of parameters in the model, such as the width of the net. In other words, to make use of a given number of noisy supervision samples, the model must have enough capacity—its size must exceed a certain threshold. We also show that a similar condition on model size applies when, instead of a data term, the PINN uses noisy samples of an initial condition $g$ at randomly chosen spatial points.

\cref{fig:intro-results} shows an illustration of the phenomenon on various PDEs, including Navier-Stokes in Figure \ref{fig:ns_0.07} and \ref{fig:ns_0.08} and Poisson in Figure \ref{fig:p_0.4} and \ref{fig:p_0.5}, and also for HJB in Figure \ref{fig:hjb_0.4} and \ref{fig:hjb_0.5}, for which we have given the theory. We provide full experimental details in \cref{sec:num_experiments}. It is to be noted that the experiments with the Poisson PDE and the Navier-Stokes PDE use unbounded activation and hence they explore beyond the ambit of the proof and the setup with the Poisson PDE resembles that structure of Theorem \ref{thm:umain} where the noise is added to the boundary condition data rather than the bulk as is done for the HJB and Navier Stokes experiments. Also, for the Navier Stokes experiments we chose a classically hard solution target which is the Taylor-Green vortex \citep{taylor1937mechanism}. For each experiment, we trained a number of PINNs at different sizes $d_N$ to solve an instance of the PDE at two different supervision noise variances $\sigma^2$. Across this diversity of setups tested, in each case it is observed that beyond a critical value of the size of the nets, the empirical error achieved at the end of training is always below $\sigma^2$. Moreover, if $d_N$ is too small, the nets fail to attain training errors below $\sigma^2$, but there is consistent increase in the performance of the PINN as $d_N$ increases upto the aforementioned critical value.


\begin{figure}[htbp!]
     \centering
     \begin{subfigure}[b]{0.48\textwidth}
         \centering
         \includegraphics[width=\textwidth]{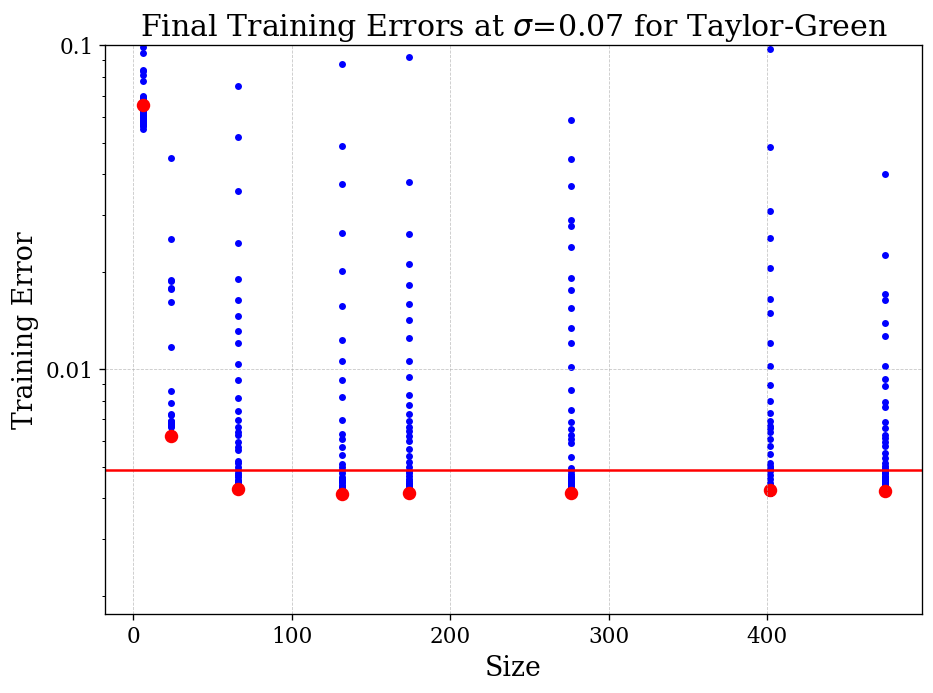}
         \caption{}
         \label{fig:ns_0.07}
     \end{subfigure}
     \hfill
     \begin{subfigure}[b]{0.48\textwidth}
         \centering
         \includegraphics[width=\textwidth]{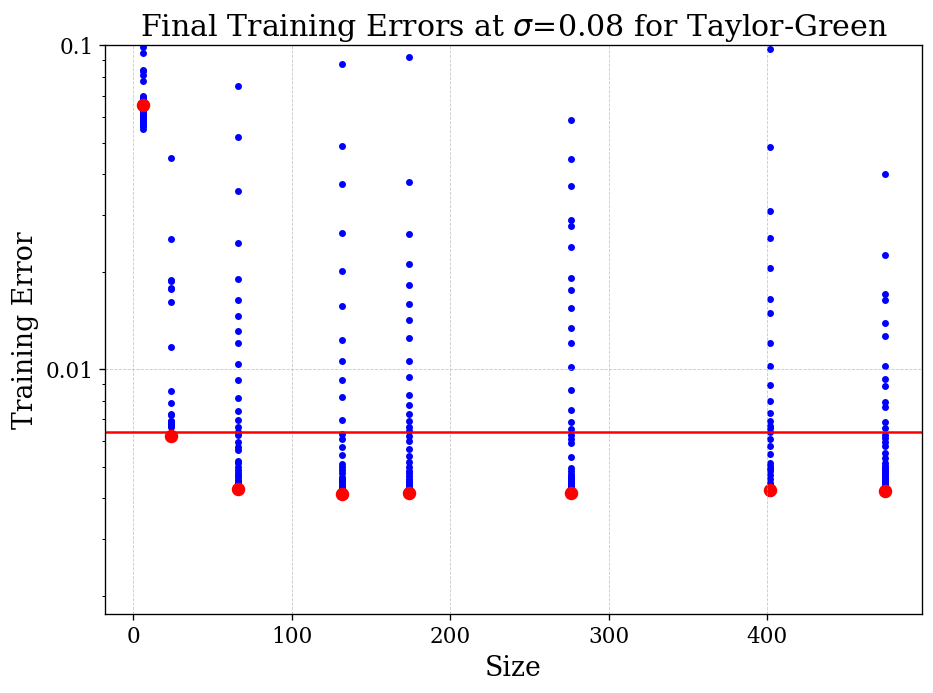}
        \caption{}
         \label{fig:ns_0.08}
     \end{subfigure}
     \begin{subfigure}[b]{0.48\textwidth}
         \centering
         \includegraphics[width=\textwidth]{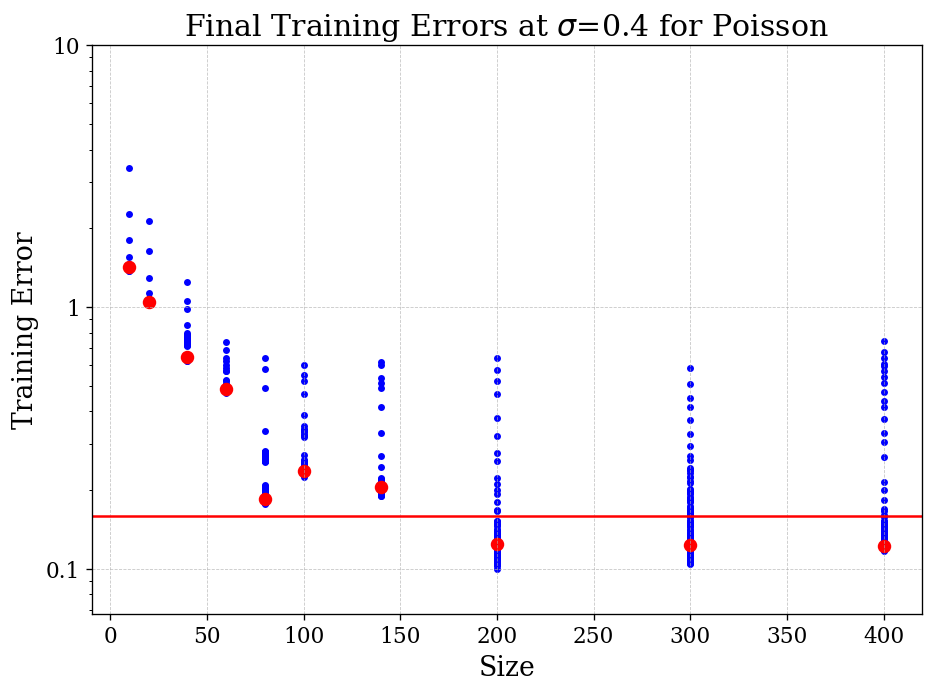}
         \caption{}
         \label{fig:p_0.4}
     \end{subfigure}
     \hfill
     \begin{subfigure}[b]{0.48\textwidth}
         \centering
         \includegraphics[width=\textwidth]{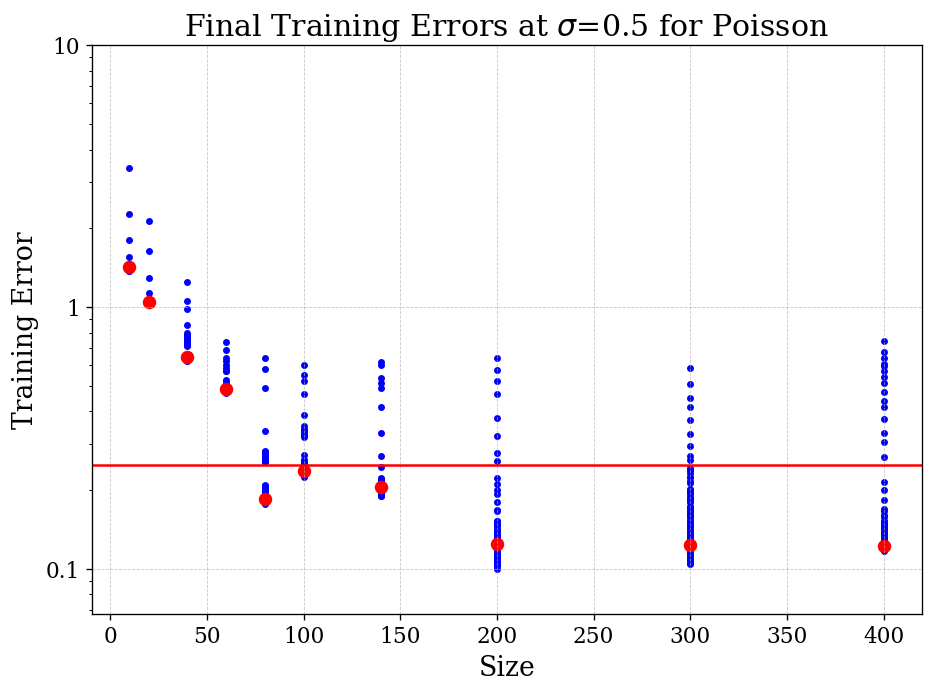}
         \caption{}
         \label{fig:p_0.5}
     \end{subfigure}
     \begin{subfigure}[b]{0.48\textwidth}
         \centering
         \includegraphics[width=\textwidth]{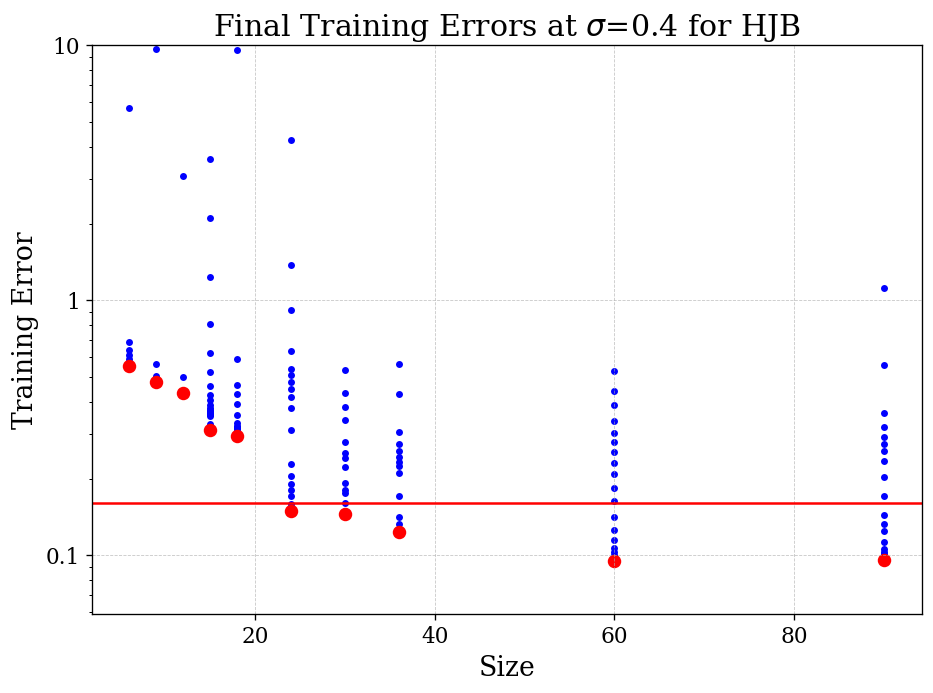}
         \caption{}
         \label{fig:hjb_0.4}
     \end{subfigure}
     \hfill
     \begin{subfigure}[b]{0.48\textwidth}
         \centering
         \includegraphics[width=\textwidth]{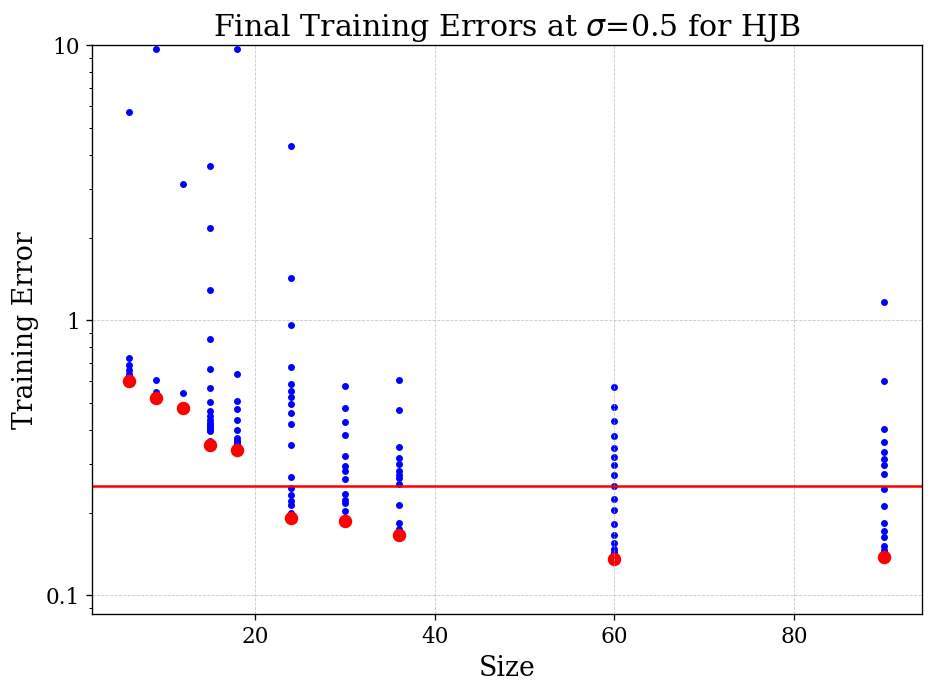}
         \caption{}
         \label{fig:hjb_0.5}
     \end{subfigure}
    \caption{Each row represents PINN experiments on different PDEs (Navier Stokes, Poisson, and HJB) for neural nets consistent with the theoretical assumptions, whose experimental details are described in Section \ref{sec:num_experiments}. For each PDE we tested at two different supervision noise variances $\sigma^2$. The vertical axis is the training error. The red dots represent the final training error, while the blue dots show the training accuracy every 100 updates, except for Navier Stokes where it is every 10. The accuracy error diminshes as the size increases before plateauing at a level below the variance given by the red line.}\label{fig:intro-results}
\end{figure}


\noindent\textbf{Method of Proof:} The proof establishes an upper bound on the probability that a good net exists within the considered class of nets, based on random sampling of the supervision data. Considering both the net's architecture and the supervision data, the existence of this bound contrasts with the requirement for this favorable event to occur with high probability, thereby imposing a necessary constraint. There are three main steps in the proof:

\begin{itemize}
	\item \cref{lastLemma} decomposes the supervised term of the PINN risk as a square of three parts: the noise in the supervision labels, their expected value, and the output of the predictor at the supervision locations. Further, invoking \cref{lastLemma} on the $\eta$-covering of the neural function class considered in the theorem, shows that the probability of finding in this set a ``good'' NN (i.e., a network with empirical risk below $\sigma^2$) has an upper bound. Up to additive constants, this upper bound is shown to equal the probability of finding a network in the same set such that the correlation between the label noise and that network's predictions at the $N_s$ supervision locations is above a certain threshold;
	\item \cref{secondLemma} analyzes the $\eta$-cover of the neural function class and demonstrates that, within this subset, the probability of finding a network with a high correlation as defined above has an upper bound that grows super-exponentially with the number of parameters and decreases exponentially with the number of supervision samples;
	\item \cref{lem:perturb} conducts an extensive calculation showing that the PINN risk changes in a non-Lipschitz but controlled manner under $\eta$-level perturbations to its weights.
\end{itemize}
These three steps are then combined to yield an upper bound on the probability of finding a predictor achieving an empirical risk below $\sigma^2$. Imposing a high probability for the existence of a good network results in the desired constraint.

\section{Related Works}

A primary question is to identify sufficient conditions under which minimizing the PINN risk provably solves the PDE. To this end, \cite{Mishra_PINN_generalization_error} considers a differential operator $\gL$ that is either linear or admits a well-defined linearization. Specifically, there exists an operator $\Bar{\gL}$ such that, for all $u$ and $v$ in the appropriate function space, $\gL(u) - \gL(v) = \Bar{\gL}_{(u,v)}(u - v)$. Under the additional assumption that this linearization operator has a bounded inverse, it is shown that the functional distance between the predictor and the true PDE solution is bounded above by the PINN risk, up to multiplicative constants.

\cite{NEURIPS2022_46f0114c} investigate the existence of NNs that can simultaneously reduce both the PINN risk and the distance to the true PDE solution. The authors show that for PDEs admitting a classical solution $u$, if there exists a NN $\gU$ that approximates the solution to arbitrary accuracy, then there also exists a $\tanh$-activated NN $u_\vw$ such that both its distance to the true solution and its PINN risk can be made arbitrarily small. Moreover, the depth of the surrogate $u_\vw$ satisfies $\text{depth}(u_\vw) = O(\text{depth}(\gU))$, and the accuracy of the PDE solution scales inversely with the width of $u_\vw$.

In \cite{devore2024}, the authors argued that for linear elliptic PDEs, the radius of the Chebyshev\footnote{A Chebyshev ball of a set $K$ is smallest ball that contains the set $K$} ball containing all solutions to the PDE corresponding to every possible source function and initial condition that can interpolate the data is a reasonable error metric. The worst case over all the functions from the set of possible continuous solutions to the PDEs and the corresponding best case over all possible training data inputs, given some budget of total data observations, gives an optimal error rate on solving the PDE that depends only on the PDE. They showed that this metric decreases polynomially with respect to the number of training observations. This work implicitly assumes noiseless observations.

The stability conditions presented in \cite{Mishra_PINN_generalization_error} and the criteria for simultaneously minimizing the PINN risk and the total error in \cite{NEURIPS2022_46f0114c} are, a priori, distinct. Whether the PINN framework for solving the HJB equation meets either of these conditions remains an open question. More broadly, identifying criteria for PDEs that would enable the simultaneous satisfaction of both conditions within the corresponding PINN setup is still an open problem.

\cite{wang2022} examine the unsupervised PINN risk for the HJB PDE. They show that even if the surrogate's boundary risk is zero and the total risk is low, this does not ensure that the surrogate solution is accurate. This is demonstrated by constructing a surrogate solution with an arbitrarily low PINN risk in the $L^2$ norm, yet which remains arbitrarily distant from the true solution when the domain has eight or more dimensions. The arguments in \cite{wang2022} do not preclude the existence of surrogates that achieve low PINN risk while remaining close to the true PDE solution.

In related work, \cite{mukherjee2024size} demonstrated that for the supervised loss of a DeepONet, evaluated on noisy labels from $N_s$ samples, to fall below a label-noise-dependent threshold, the output dimension $q$ of the branch/trunk network must satisfy $\smash{q = \Omega(N_s^{1/4})}$. However, this study did not provide a lower bound on the total number of trainable parameters required to achieve the same level of performance. Additionally, \cite{bubecksellkerobustness} showed for the first time that for a class of NNs with depth $D$ and $d_N$ parameters, in order for a member of this class to achieve a squared loss empirical risk on $N_s$ noisy samples in $d$ dimensions below the noise variance, the Lipschitz constant of the predictor $f$ must satisfy $\smash{\text{Lip}(f)=\tilde{\Omega}(\sqrt{{N_sd}/(Dd_N)})}$.

\section{Mathematical Setup of the HJB PINN Loss}

We say that a scalar function $u : \R^{n+1} \rightarrow \R$ on the domain $[-B,B]^n \times [0,T]$ for some $B, T > 0$ satisfies the HJB PDE in the strong sense if
\begin{align}\label{def:hjb}\setlength\abovedisplayskip{6pt}\setlength\belowdisplayskip{6pt}
\nonumber \tilde{\gL}_{\rm HJB}(u) = \partial_tu-\nabla^2 u+\norm{\nabla u}_2^2=0 \quad& ~\forall (\vx,t) \in \gD = [-B,B]^n\times[0,T] \\
\tilde{\gB}_{\rm HJB}(u) = u(\vx,0) =g(\vx) \quad& ~\forall \vx \in [-B,B]^n.
\end{align}
Here, $\smash{\nabla^2 u = \sum_{i=1}^n \pdv[2]{u}{x_i} = \sum_{i=1}^n \partial_i^2 u}$ is the spatial Laplacian of $u$ and $\nabla u = \left ( \partial_1 u,\ldots,\partial_n u \right ) \in \R^n$ is the spatial gradient of $u$. For example, an exact solution that has been used as a test case for benchmarking PINN experiments \citep{hu2023exactsolutions} corresponds to $g(\vx)=\norm{\vx}^2$ is,
\begin{align}\setlength\abovedisplayskip{6pt}\setlength\belowdisplayskip{6pt}\label{eq:HJBsol}
u(\vx,t)=\frac{\norm{\vx}^2}{1+4t}+\frac{n}{2}\log(1+4t).
\end{align}

To approximately solve the above PDE with arbitrary initial conditions $g$ in the PINN formalism, we consider the following class of predictors. The specific function class $\gH$ we consider consists of NNs with a bounded activation function in the final layer, without constraints on activations in the preceding layers.

\begin{definition}[Class of NNs]\label{def:nets}
Let $f$ be a non-linear activation function (or indeed a full NN that is frozen) that satisfies:
\begin{itemize}
\item (Bounded and bounded gradient) $\norm{f(\vz)}\leq C_1$, $\norm{\nabla f(\vz)}\leq C_2$ for all $\vz \in \R^k$;
\item (Lipschitz) $\norm{f(\vz_1)-f(\vz_2)}\leq L_1\norm{\vz_1-\vz_2}$ for all $\vz_1, \vz_2 \in \R^{k}$;
\item (Lipschitz gradient) $\norm{\nabla f(\vz_1)-\nabla f(\vz_2)}\leq L_2\norm{\vz_1-\vz_2}$ for all $\vz_1, \vz_2 \in \R^{k}$\\or  $\norm{Hf(\vz)}_F\leq C_3$ for all $\vz \in \R^k$.
\end{itemize}
Let $\gH$ be the set of functions
$$\setlength\abovedisplayskip{6pt}\setlength\belowdisplayskip{6pt}
    \R^n \times \R \ni (\vx,t) \mapsto h_{\vw}(\vx,t) = f(\mW_1 \vx + \vw_2 t) \in \R,\quad \mW_1 \in \R^{k \times n}, \vw_2 \in \R^k
$$
with parameters $\vw=(\mW_1,\vw_2)$ that satisfy $\norm{\vw}_2\leq W$.
\end{definition}

The number of trainable parameters of a function in $\gH$ is $d_N = k(n+1)$. By compactness of $\smash{\{ \vw \in \R^{d_N} \mid \norm{\vw}_2 \leq W\}}$, for any $\eta>0$, there exists a ``$\eta$-cover'' of the function space, $\gH_\eta$, such that for all $h_\vw\in\gH$ there exists $\smash{h_{\vw_{{\eta}/{2}}}\in\gH_\eta}$ with $\|\vw-\vw_{{\eta}/{2}}\|\leq {\eta}/{2}$. An example of $f$ is $ \vz \mapsto \ip{\va}{\sigma(\vz)} $, where $\va \in \R^k$ is a fixed vector and $\sigma$ is a smooth bounded nonlinear function like sigmoid or tanh, acting point-wise. We show in \cref{sec:depth2} that when $\sigma$ is sigmoid activation function, one can choose $C_1=\norm{\va}_1$, $C_2=\frac{1}{4}\norm{\va}$, and $C_3=\frac{1}{6\sqrt{3}}\norm{\va}$.

We can now introduce our loss function for the PINN framework.

\begin{definition}\label{lossfunction}
For $B, T > 0$, let the computational domain be $\gD = [-B,B]^n \times [0,T]$. Assume we are given two sets of collocation points, $\{(\vx_{ri}, t_{ti}) \in \gD \mid i=1,\ldots,N_r\}$ and $\{\vx_{0j} \in [-B,B]^n \mid j=1,\ldots,N_0\}$, along with supervision data $\{((\vx_{si}, t_{si}), y_{si}) \in \gD \times [-M,M] \mid i = 1,\ldots,N_s\}$, where $M > 0$ bounds the solution samples. The corresponding empirical loss function for regularization parameters $\lambda_0,\lambda_s>0$ is:
\begin{align}\label{def:hjbloss}\setlength\abovedisplayskip{6pt}\setlength\belowdisplayskip{6pt}
\nonumber \hat{R}(h_\vw)= &\underbrace{\frac{1}{N_r}\sum_{i=1}^{N_r}\left(\partial_th_\vw(\vx_{ri}, t_{ti})-\nabla^2h_\vw(\vx_{ri}, t_{ti})+\norm{\nabla h_\vw(\vx_{ri}, t_{ti})}_2^2\right)^2}_{\text{PDE Residual Loss}}\\
&\qquad+ \underbrace{\frac{\lambda_0}{N_0}\sum_{j=1}^{N_0}\left(h_\vw(\vx_{0j},0)-g(\vx_{0j})\right)^2}_{\text{Initial Condition Loss}} + \underbrace{\frac{\lambda_s}{N_s}\sum_{k=1}^{N_s}(h_\vw(\vx_{sk},t_{sk})-y_{sk})^2}_{\text{Supervised Loss With Noisy Solution Samples}}.
\end{align}
\end{definition}

\section{Main Results}

We now have all the necessary prerequisites to state the main result as follows:

\begin{theorem}\label{thm:main}
Consider the data and PINN loss function for the HJB PDE as given in \cref{lossfunction}. Let the initial condition $g$ be bounded by $\|g\|_\infty \leq G$, and let $\sigma^2=\frac{1}{N_s}\sum_{i=1}^{N_s}\E[z_i^2]$, where $z_i=y_{si}-\E[y_{si}|(x_{si},t_{si})]$, with supervision data sampled as $\{(y_i,(\vx_{si},t_{si})) \mid i=1,\ldots,N_s \}$. Then there exists a constant $a > 0$ such that for all sufficiently small $\eta > 0$, and for all $\delta \in (0,1)$, the following holds. If with probability at least $1-\delta$ (with respect to the sampling of the supervision data), there exists $h_\vw \in \gH$ such that
\begin{equation}\label{eq:maintheorem-losscondition}\setlength\abovedisplayskip{6pt}\setlength\belowdisplayskip{6pt}
\hat{R}(h_\vw) \leq \sigma^2-a\eta,
\end{equation}
then the size of the model ($d_N$) is lower bounded by the number of supervision samples ($N_s$):
\begin{equation}\label{eq:maintheorem-lowerbound}\setlength\abovedisplayskip{6pt}\setlength\belowdisplayskip{6pt}
d_N \cdot \left ( \ln\left(d_N\right)+\ln\left(\frac{4W^2}{\eta^2}\right) \right )  \geq\frac{N_s\eta^2}{144M^4}-2\ln\left(\frac{4}{1-\delta}\right),
\end{equation}
where $W$ is from \cref{def:nets}.
\end{theorem}

The quantity $a \eta$ parametrizes how much we are pushing the empirical risk of the PINN below the noise-dependent threshold $\sigma^2$, where the constant $a >0$ depends on the specifics of the loss function, i.e., $\lambda_s$, the collocation points in the PINN loss (\cref{def:hjbloss}) and the choice of $\gH$ (\cref{def:nets}). The proof of \cref{thm:main} is given in \cref{proof:main}. 

\begin{remark}[Good solvers]
Predictors $h_\vw$ that satisfy the condition in \cref{eq:maintheorem-losscondition} can naturally be referred to as ``good PDE solvers'', in the sense that their PINN empirical loss is below the label noise threshold. Therefore, the above theorem establishes a necessary condition for the existence of such good solvers.
\end{remark}

\begin{remark}[Extension to general $\eta$]
\cref{thm:main}, as stated, is limited to the small $\eta > 0$ regime that lead to easily interpretable constraints. A study of the proof reveals that it extends to more general $\eta$ with more complicated final bounds on $d_N$.
\end{remark}

We can also consider the PINN loss in \cref{lossfunction} with $\lambda_s = 0$, corresponding to the absence of supervised data. In this case, consider that the initial condition $g$ is observed noisily through samples $y_{0j}$ at randomly sampled locations $(\vx_{0j}, 0)$ for $j = 1, \ldots, N_0$. A similar lower bound as in \cref{thm:main} still holds for the following corresponding loss function:
\begin{align}\label{def:uhjbloss}\setlength\abovedisplayskip{6pt}\setlength\belowdisplayskip{6pt}
\nonumber \hat{R}_u(h_\vw)= &\underbrace{\frac{1}{N_r}\sum_{i=1}^{N_r}\left(\partial_th_\vw(\vx_{ri}, t_{ti}){-}\nabla^2h_\vw(\vx_{ri}, t_{ti}){+}\norm{\nabla h_\vw(\vx_{ri}, t_{ti})}_2^2\right)^2}_{\text{PDE Residual Loss}}\\
&+ \underbrace{\frac{\lambda_0}{N_0}\sum_{j=1}^{N_0}\left(h_\vw(\vx_{0j},0){-}y_{0j}\right)^2}_{\text{Initial Condition Loss With Noise}}.
\end{align}
The proof of the following theorem is omitted, as it follows the same lines as the proof of \cref{thm:main}.

\begin{theorem}\label{thm:umain}
Consider the data and PINN loss function $\hat{R}_u$ for the HJB PDE as given in \cref{def:uhjbloss}. Let $\abs{y_{0j}}\leq G$ for all $j$, and let $\sigma^2=\frac{1}{N_0}\sum_{j=1}^{N_0}\E[z_j^2]$, where $z_j=y_{0j}-\E[y_{0j}|(\vx_{0j},0)]$, with the initial condition sampled as $\{(y_{0j},(\vx_{0j},0))|j=1,\ldots,N_0\}$. Then there exists a constant $c >0$ such that for all sufficiently small $\eta > 0$, and for all $\delta \in (0,1)$, the following holds. If with probability at least $1-\delta$ (with respect to the sampling of the initial condition), there exists $h_\vw\in\gH$ such that
\begin{align}\label{eq:umain-losscondition}\setlength\abovedisplayskip{6pt}\setlength\belowdisplayskip{6pt}
\nonumber\hat{R}_u(h_\vw)\leq \sigma^2-c\eta,
\end{align}
then $d_N$ is lower bounded by the number of initial condition samples ($N_0$):
    \begin{align} \setlength\abovedisplayskip{6pt}\setlength\belowdisplayskip{6pt}
    d_N \cdot \left ( \ln\left(d_N\right)+\ln\left(\frac{4W^2}{\eta^2}\right) \right ) \geq\frac{N_0\eta^2}{144G^4}-2\ln\left(\frac{4}{1-\delta}\right).
    \end{align}
\end{theorem}

\cref{thm:main,thm:umain} imply that noisy supervision is not a ``free lunch'', whether it involves access to noisy samples of the initial condition or solution samples. In either case, the model size must exceed a critical threshold for these noisy samples to result in a better PINN predictor.

\subsection{Proof of the Main Result}\label{proof:main}

To prove the main result, we require the following three lemmas, which are proven in \cref{sec:lemmas_pfs}. We begin with stating a general property that can be proven for any predictor for which the HJB PINN empirical risk is definable.

\begin{lemma}\label{lastLemma}
Corresponding to supervision data --- for the PINN risk $\hat{R}$ in as in \cref{lossfunction} --- sampled as $\{(y_i,(\vx_{si},t_{si})) \mid i=1,\ldots,N_s \}$, define the random variables $z_i=y_{si}-\E[y_{si}|(x_{si},t_{si})]$. Let $\gP$ be any class  of parameterized functions $p_\vw : \R^{n +1} \rightarrow \R$ for which $\hat{R}$ can be defined. Then for all $\eta,\mu>0$,
\begin{align*}\setlength\abovedisplayskip{6pt}\setlength\belowdisplayskip{6pt}
\Pr\left(\exists ~p_\vw\in\gP \mid \hat{R}(p_\vw)\leq\sigma^2-\mu\eta\right)&\leq2\exp\left(-\frac{N_s\eta^2}{288M^4}\right)\\
&\quad+\Pr\left(\exists ~p_\vw\in\gP \,\Big| \frac{1}{N_s}\sum_{i=1}^{N_s}p_\vw(\vx_{si},t_{si})z_i\geq     S\right),
\end{align*}
where $S=\frac{1}{2\lambda_s}\left(\mu\eta - \sigma^2\right) + \left(\frac{1}{2}\sigma^2 - \frac{1}{4}\eta\right)$.
\end{lemma}

Specific to the function class $\gH_\eta$, as given in \cref{def:nets}, we upper bound the event in the right-hand side above, about the correlation between the noise in the supervision data and predictions on the same.

\begin{lemma}\label{secondLemma}
The following probabilistic bound holds:
\begin{align*}\setlength\abovedisplayskip{6pt}\setlength\belowdisplayskip{6pt}
\Pr\left(\exists ~h_{\vw_{\frac{\eta}{2}}}\in\gH_\eta\Big|\frac{1}{N_s}\sum_{i=1}^{N_s}h_{\vw_\frac{\eta}{2}}(\vx_{si},t_{si})z_i\geq S\right)&\leq2\left(\frac{2W\sqrt{d_N}}{\eta}\right)^{d_N}\exp{\frac{-N_sS^2}{2\cdot(8MC_1)^2}} \\
&\quad+ 2\exp{\frac{-N_sS^2}{32M^2C_1^2}},
\end{align*}
where $S=\frac{1}{2\lambda_s}\left(\mu\eta - \sigma^2\right) + \left(\frac{1}{2}\sigma^2 - \frac{1}{4}\eta\right)$.
\end{lemma}

Lastly, we need the following lemma which relates the risk of the predictors in the function class considered and predictors whose weights are $\frac{\eta}{2}$ away from it.

\begin{lemma}[Perturbation bound for PINN empirical risk for HJB]\label{lem:perturb}
Recall $g$ as defined in the HJB \cref{def:hjb}, and let $G = \|g(\vx)\|_\infty$. Referring to the definition of the PINN loss $\hat{R}$ in \cref{def:hjbloss}, we have
\begin{equation}\label{eq:perturb-bound}\setlength\abovedisplayskip{6pt}\setlength\belowdisplayskip{6pt}
\hat{R}(h_{\vw_{{\eta}/{2}}})\leq\hat{R}(h_\vw) + a_1\eta + a_2\eta^2 + a_3\eta^3 + a_4\eta^4\quad \forall \eta>0,
\end{equation}
where
\begin{align*}
b_1&=2\left((WC_3V(n)+C_2)\left(1+2WC_2\right)+C_3nkW\right)\left(WC_2(1+WC_2)+C_3nkW^2\right)\\
a_1&=\lambda_sC_2 V(n)\left(C_1+M\right) + \lambda_0C_2\sqrt{n}B\left(C_1+G\right) + b_1\\
a_2&=\frac{1}{4}\left(C_3nk+(WC_3V(n)+C_2)^2\right)\left(WC_2(1+WC_2)+C_3nkW^2\right)\\
&\quad+\left((WC_3V(n)+C_2)\left(\frac{1}{2}+WC_2\right)+C_3nkW\right)\\
&\quad\quad\times\left(C_2\left(\frac{1}{2}+W\left(C_3V(n)+C_2^2+1\right)\right)+W\left(\frac{1}{2}C_3V(n)+C_3nk\right)\right)\\
a_3&=\frac{1}{4}\left(C_3nk+WC_3V(n)+C_2\right)\left((WC_3V(n)+C_2)\left(\frac{1}{2}+WC_2\right)+C_3nkW\right)\\
&\quad+\frac{1}{4}\left(C_3nk+(WC_3V(n)+C_2)^2\right)\Big(C_2\left(\frac{1}{2}+W\left(C_3V(n)+C_2^2+1\right)\right)\\
&\quad+W\left(\frac{1}{2}C_3V(n)+C_3nk\right)\Big)\\
a_4&=\frac{1}{16}\left(C_3nk+WC_3V(n)+C_2\right)\left(C_3nk+\left(WC_3V(n)+C_2\right)^2\right).
\end{align*}
Here, $B$ and $T$ are the bounds on the space-time domain for the HJB equation, and $V(n)=\sqrt{n}B+T$. The constants $C_1, C_2$, $C_3$, and $W$ is defined in \cref{def:nets}.
\end{lemma}

The terms $b_1$, $a_2$, $a_3$, and $a_4$ arise from the PDE residual loss components, while the other two terms in $a_1$, i.e., $\lambda_s C_2 V(n)\left(C_1 + M\right)$ and $\lambda_0 C_2 \sqrt{n} B\left(C_1 + G\right)$, come from the supervised and initial condition loss terms, respectively. Note that $a_1 = 2C_3^2 k^2 W^3 n^2 + O(n^{3/2})$, where the leading coefficient, $(C_3 k)^2 W^3$, is due to the Laplacian term in the PDE loss. All other powers of $n$ depend on the nonlinear term of the PDE loss.

Given these lemmas, we can now prove \cref{thm:main}.

\begin{proof}[Proof of \cref{thm:main}]
We first apply \cref{lastLemma} with $\gP = \gH_\eta$ and $\mu=s>0$ to see that
\begin{align*}
\setlength\abovedisplayskip{6pt}\setlength\belowdisplayskip{6pt}
&\Pr\left(\exists ~h_{\vw_\frac{\eta}{2}}\in\gH_\eta|\hat{R}(h_{\vw_\frac{\eta}{2}})\leq\sigma^2-s\eta\right)\\
&\leq2\exp\left({-}\frac{N_s\eta^2}{288M^4}\right){+}\Pr\left(\exists h_{\vw_\frac{\eta}{2}}{\in}\gH_\eta|\frac{1}{N_s}\sum_{i=1}^{N_s}h_{\vw_\frac{\eta}{2}}(x_{si},t_{si})z_i\geq
    \frac{1}{2\lambda_s}\left(s\eta - \sigma^2\right) {+} \left(\frac{1}{2}\sigma^2 {-} \frac{1}{4}\eta\right)\right)\\
&\leq2\exp\left(-\frac{N_s\eta^2}{288M^4}\right)+2\left(\frac{2W\sqrt{d}}{\eta}\right)^d\exp{\frac{-N_sS^2}{2\cdot(8MC_1)^2}} + 2\exp{\frac{-N_sS^2}{32M^2C_1^2}},
\end{align*}
where $S=\frac{1}{2\lambda_s}\left(s\eta - \sigma^2\right) + \left(\frac{1}{2}\sigma^2 - \frac{1}{4}\eta\right)$. Here, we have applied \cref{secondLemma} to arrive at the final inequality.
		
We now recall \cref{lem:perturb} and the quantities $a_1, a_2, a_3,a_4$ such that
$$\setlength\abovedisplayskip{6pt}\setlength\belowdisplayskip{6pt}
\hat{R}(h_{\vw_{\frac{\eta}{2}}})\leq\hat{R}(h_\vw) + a_1\eta + a_2\eta^2 + a_3\eta^3 + a_4\eta^4.
$$
With respect to the random sampling of the training data, define the following two events:
\begin{align*}\setlength\abovedisplayskip{6pt}\setlength\belowdisplayskip{6pt}
&\mE_1=\{\exists~h_\vw\in\gH \mid \hat{R}(h_\vw)\leq\sigma^2-\left((s+a_1)\eta + a_2\eta^2 + a_3\eta^3 + a_4\eta^4\right)\}\\
&\mE_2=\{\exists~h_{\vw_\frac{\eta}{2}}\in\gH_\eta\mid\hat{R}(h_{\vw_\frac{\eta}{2}})\leq\sigma^2- s\eta\}.
\end{align*}
The above inequality shows that $\mE_1$ implies $\mE_2$, and hence $\Pr(\mE_1)\leq\Pr(\mE_2)$ and 
    \begin{align}\label{eq:upbound}\setlength\abovedisplayskip{6pt}\setlength\belowdisplayskip{6pt}
        \nonumber &\Pr\left(\exists~h_\vw\in\gH|\hat{R}(h_\vw)\leq\sigma^2-s\eta-(a_1\eta + a_2\eta^2 + a_3\eta^3 + a_4\eta^4)\right)\\
        \nonumber &\leq\Pr\left(\exists~h_{\vw_\frac{\eta}{2}}\in\gH_\eta|\hat{R}(h_{\vw_\frac{\eta}{2}})\leq\sigma^2-s\eta\right)\\
         &\leq 2\exp\left(-\frac{N_s\eta^2}{288M^4}\right)+2\left(\frac{2W\sqrt{d_N}}{\eta}\right)^{d_N}\exp{\frac{-N_sS^2}{2(8MC_1)^2}} + 2\exp{\frac{-N_sS^2}{32(MC_1)^2}}.
    \end{align}
Without loss of generality, we assume that $M>1$. Let $\delta \in (0,1)$. Further, we note that we may choose the constant $a>0$ such that
$
(s+a_1)\eta + a_2\eta^2 + a_3\eta^3 + a_4\eta^4\leq a\eta
$
for the $\eta$ small enough. Hence, we can define the event that the theorem requires to occur with high probability,
$$\setlength\abovedisplayskip{6pt}\setlength\belowdisplayskip{6pt}
\mE_0=\{\exists~h_\vw\in\gH \mid \hat{R}(h_\vw)\leq\sigma^2- a \eta \}.
$$
 It follows that $\Pr[\mE_0] \leq \Pr[\mE_1]$ and it follows from \cref{eq:upbound} that if $\mE_0$ has to happen with probability at least $1-\delta$ then necessarily the following must hold,
$$\setlength\abovedisplayskip{6pt}\setlength\belowdisplayskip{6pt}
1-\delta\leq 2\exp\left(-\frac{N_s\eta^2}{288M^4}\right)+2\left(\frac{2W\sqrt{d_N}}{\eta}\right)^{d_N}\exp{\frac{-N_sS^2}{2(8MC_1)^2}} + 2\exp{\frac{-N_sS^2}{32(MC_1)^2}}.
$$
This inequality simplifies to:
\begin{align*}\setlength\abovedisplayskip{6pt}\setlength\belowdisplayskip{6pt}
\frac{1-\delta}{2}&\leq \exp\left(-\frac{N_s\eta^2}{288M^4}\right)+\left(\frac{2W\sqrt{d_N}}{\eta}\right)^{d_N}\exp{\frac{-N_sS^2}{32(MC_1)^2}} + \exp{\frac{-N_sS^2}{32(MC_1)^2}}\\
&\leq \exp\left(-\frac{N_s\eta^2}{288M^4}\right)+\left(1+\left(\frac{2W\sqrt{d_N}}{\eta}\right)^{d_N}\right)\exp{\frac{-N_sS^2}{32(MC_1)^2}}\\
&\leq \exp\left(-\frac{N_s\min\left\{\frac{\eta^2}{9},\frac{S^2}{C_1^2}\right\}}{32M^4}\right)+\left(1+\left(\frac{2W\sqrt{d_N}}{\eta}\right)^{d_N}\right)\exp{-\frac{N_s\min\left\{\frac{\eta^2}{9},\frac{S^2}{C_1^2}\right\}}{32M^4}}\\
&=\left(2+\left(\frac{2W\sqrt{d_N}}{\eta}\right)^{d_N}\right)\exp{-\frac{N_s\min\left\{\frac{\eta^2}{9},\frac{S^2}{C_1^2}\right\}}{32M^4}}
\end{align*}
We now consider suppose that $s$ is in one of the following three regimes: if $\lambda_s < 1$ then $s \in (0, {\lambda_s}/{2})$, or if $\lambda_s > 1$ then $s > {\lambda_s}/{2}$, or if $\lambda_s = 1$ then $s \geq {2C_1}/{3} + {1}/{2}$. This will ensure that $S \neq 0$ and hence the upper bound above is non-trivial. We rename $S = S_\eta$ to make the $\eta$ dependence of $S$ explicit. Further simplifying the above we obtain 
\begin{equation}\label{eq:genthetabound}\setlength\abovedisplayskip{6pt}\setlength\belowdisplayskip{6pt}
N_s\leq\frac{32M^4}{\min\left\{\frac{\eta^2}{9},\frac{S_\eta^2}{C_1^2}\right\}}\ln\left(\frac{2}{1-\delta}\cdot\left(2+\left(\frac{2W\sqrt{d_N}}{\eta}\right)^{d_N}\right)\right).
\end{equation}
Towards making this bound more intuitive, note that
\begin{align}\label{eq:mindiff}\setlength\abovedisplayskip{6pt}\setlength\belowdisplayskip{6pt}
    \nonumber\frac{S_\eta^2}{C_1^2} - \frac{\eta^2}{9}  &= \frac{1}{C_1^2} \cdot \left \{ \frac{1}{2\lambda_s}\left(s\eta - \sigma^2\right) + \left(\frac{1}{2}\sigma^2 - \frac{1}{4}\eta\right) \right \}^2 - \frac{\eta^2}{9}\\
   &=\left(\frac{1}{C_1^2}\left(\frac{s}{2\lambda_s}-\frac{1}{4}\right)^2-\frac{1}{9}\right)\eta^2+\frac{\sigma^2}{C_1^2}\left(\frac{s}{2\lambda_s}-\frac{1}{4}\right)\left(1-\frac{1}{\lambda_s}\right)\eta+\frac{\sigma^4}{2C_1^2}\left(1-\frac{1}{\lambda_s}\right)^2.
\end{align}
We now consider the different cases for $\lambda_s.$

\paragraph{Case 1: $\lambda_s\neq1$.}
The two corresponding regimes considered above ensure that $S_\eta \neq0$. Further, in these cases \cref{eq:mindiff} has a non-negative constant term, i.e, $\frac{\sigma^4}{2C_1^2}\left(1-\frac{1}{\lambda_s}\right)^2 >0$ and thus in the limit of $\eta\rightarrow0$, $\frac{\eta^2}{9}<\frac{S_\eta^2}{C_1^2}$. Hence, we have
$$\setlength\abovedisplayskip{6pt}\setlength\belowdisplayskip{6pt}
N_s\leq\frac{288M^4}{\eta^2}\ln\left(\frac{2}{1-\delta}\cdot\left(2+\left(\frac{2W\sqrt{d_N}}{\eta}\right)^{d_N}\right)\right).
$$
For convenience, define $N_c={288M^4}/{\eta^2}$ so that
$$\setlength\abovedisplayskip{6pt}\setlength\belowdisplayskip{6pt}
N_s\leq N_c\left(\ln\left(1+\frac{1}{2}\left(\frac{2W\sqrt{d_N}}{\eta}\right)^{d_N}\right)+\ln\left(\frac{4}{1-\delta}\right)\right).
$$
For large enough $d_N$, this yields the inequality
$$\setlength\abovedisplayskip{6pt}\setlength\belowdisplayskip{6pt}
N_s\leq N_c\left(\frac{d_N}{2}\ln\left(d_N\left(\frac{2W}{\eta}{}\right)^2\right)+\ln\left(\frac{4}{1-\delta}\right)\right),
$$
which rearranges to \cref{eq:maintheorem-lowerbound}.

\paragraph{Case 2: $\lambda_s = 1$.} In this case, since $s>\frac{2C_1}{3}+\frac{1}{2}$, we have $ \frac{S_\eta^2}{C_1^2} - \frac{\eta^2}{9} = \left(\frac{1}{C_1^2} \cdot \left \{ \frac{s}{2} - \frac{1}{4}\right \}^2 - \frac{1}{9}\right) \eta^2$, which is positive for small $\eta$. We can now argue as in the first case.
\end{proof}

\section{Experimental Details}
\label{sec:num_experiments}
We aim to illustrate Theorem \ref{thm:main} and \ref{thm:umain} on various PDEs, including HJB, for which we give the theory and also for the Navier-Stokes and the Poisson PDE.

\subsection{ Experiment Setup with the Navier Stokes PDE}
We recall the 2D incompressible Navier Stokes Equations:
\begin{equation}\label{eqn:ns}
\begin{aligned}
\partial_t \vu + (\vu\cdot\nabla)\vu + \frac{1}{\rho}\nabla p &= \nu\nabla^2\vu\\ 
\nabla\cdot\vu&=0
\end{aligned}
\end{equation}

where $\vu$ is the two-dimensional velocity field and $p$ is the pressure field and $\rho, \nu >0$ are density and viscosity constants. We choose to test our result on the Taylor-Green vortex solution \citep{taylor1937mechanism} to the Navier-Stokes equations with dimension $d=2$. The true solution is given by:
\begin{align*}\setlength\abovedisplayskip{6pt}\setlength\belowdisplayskip{6pt}
    \vu &= (u,v)\\
    u(x,y,t)&=\sin(x)\cos(y)e^{-2\nu t}\\
    v(x,y,t)&=-\cos(x)\sin(y)e^{-2\nu t}\\
    p(x,y,t)&=-\frac{\rho}{4}\left(\cos(2x)+\cos(2y)\right)e^{-4\nu t}
\end{align*}
A simple architecture has been chosen that satisfies \cref{def:nets}, specifically $h_\vw(x,y,t)=\ip{\va}{\tanh(\mW_1\vz)}$, where $\vz=(x,y,t)\in\R^3,\mW_1\in\R^{k \times 3}$ and $\tanh(x)$ is applied element-wise. Each coordinate in $\va\in\R^k$ has been sampled from a Gaussian distribution with mean $0$ and variance $1$.

The space-time domain is set to the unit cube $(x,y,t)\in[0,1]^3$ and the loss function is defined analogously to Definition \ref{def:hjbloss}, where the parameters are chosen as $N_s=3276$, $\lambda_0=0.3$ and $\lambda_s=0.5$.

The PINNs were trained for $3000$ epochs using mini-batching with 100 batches each using PyTorch's Adam algorithm with a learning rate of $1\cdot10^{-3}$ and the training error was recorded. We varied the NNs by increasing their widths, thereby increasing the number of trainable parameters $d_N$ as $d_N=3k$. The results are shown in Figure \ref{fig:ns_0.07} and \ref{fig:ns_0.08}.

\subsection{Experiment Setup with the Poisson PDE}
We recall the Poisson equation with a source function $f(\vx)$:
\begin{equation}\label{eqn:poisson}
\begin{aligned}
- \Delta u &= f(\vx ) ~\forall \vx \in \Omega   
\end{aligned}
\end{equation} 
We choose to test our secondary result described in Theorem \ref{thm:umain} on a manufactured solution to the 2D Poisson equation. So instead of a noisy supervision term on the bulk, we introduced noise into the boundary data. The manufactured solution is given by:
\begin{align*}\setlength\abovedisplayskip{6pt}\setlength\belowdisplayskip{6pt}
    u(x,y)=x-x^2+y-y^2,\quad f(\vx)=-4
\end{align*}
A simple architecture has been chosen that satisfies \cref{def:nets}, specifically $h_\vw(x,y)=\ip{\va}{\tanh(\mW_1z)}$, where $z=(x,y)\in\R^2,\mW_1\in\R^{k \times 2}$. Each coordinate in $\va\in\R^k$ has been sampled from a Gaussian distribution with mean $0$ and variance $1$.

The domain is set to the unit square $(x,y)\in[0,1]^2$ and the other parameters are chosen as $N_0=596$ and $\lambda_0=1$. The PINNs were trained for $20000$ updates using PyTorch's Adam algorithm with a learning rate of $1\cdot10^{-2}$ and the training error at the last update was recorded. We varied the NNs by increasing their widths, thereby increasing the number of trainable parameters $d_N$ as $d_N=2k$ (the number of trainable parameters). The results are shown in Figure \ref{fig:p_0.4} and \ref{fig:p_0.5}.

\subsection{Experiment Setup with the HJB PDE}

We choose to test our result on a known solution to HJB detailed in \cref{eq:HJBsol} with $n=2$ and initial condition $g(x,y)=x^2+y^2$. The true solution given by:
\begin{align*}\setlength\abovedisplayskip{6pt}\setlength\belowdisplayskip{6pt}
    u(x,y,t)=\frac{x^2+y^2}{1+4t}+\log(1+4t)
\end{align*}
A simple architecture has been chosen that satisfies \cref{def:nets}, specifically $h_\vw(x,y,t)=\ip{\va}{\tanh(\mW_1z+\vw_2t)}$, where $z=(x,y)\in\R^2,\mW_1\in\R^{k \times 2},\vw_2\in\R^k$. Each coordinate in $\va\in\R^k$ has been sampled from a Gaussian distribution with mean $3$ and variance $1$.

The PINNs have been trained using the empirical error defined in \cref{def:hjbloss}. The domain is set to the unit cube $(x,y,t)\in[0,1]^3$ and the other parameters are chosen as $N_s=3276$, $\lambda_s=0.5$, and $\lambda_0=0.3$. The NNs were trained for $20000$ updates using PyTorch's Adam algorithm with a learning rate of $1\cdot10^{-3}$ and the training error at the last update was recorded. We varied the NNs by increasing their widths, thereby increasing the number of trainable parameters $d_N$ as $d_N=3k$. The results are shown in Figure \ref{fig:hjb_0.4} and \ref{fig:hjb_0.5}.
\section{Conclusion}

We provided a detailed analysis of the semi-supervised PINN risk for the Hamilton--Jacobi--Bellman PDE with noisy labels. We proved that the minimum number of trainable parameters $d_N$ required to achieve an empirical PINN risk $\mathcal{O}(\eta)$ below a label-noise-dependent threshold must scale with the number of samples $N_s$ such that $d_N \log d_N\gtrsim N_s\eta^2$. In practical terms, our proof hints that simply adding more noisy data will not improve a PINN's accuracy unless the model size is also increased sufficiently and our bound quantifies this requirement. Our lower bound applies in two distinct setups: (1) where $N_s$ noisy observations of the solution serve as a supervised penalty, and (2) where the PINN operates in an entirely unsupervised manner with $N_s$ noisy boundary condition observations. To our knowledge, this is the first result providing a lower bound on NN size required to achieve specified performance levels for PINNs.

Our experiments demonstrate the above phenomenon on other PDEs and activations than what our proof considers, like a setup with the Navier-Stokes PDE. Our concentration-inequality-based approach is quite general and could be applied to other PDEs and model classes, suggesting a broader framework for determining NN size constraints for scientific applications.   Hence we are led to believe that what we have uncovered is not an isolated phenomenon – that many other complex PDEs likely exhibit similar requirements on model size for learning from noisy data. Exploring such PDEs and extending our bounds to more general classes of PINNs is an exciting direction for future work. This work, therefore, opens up a new direction for rigorous investigations into NN scaling requirements in solving complex physical problems. Not only does this help in theoretically understanding PINNs, but also in employing the correct network architectures.

Further research could also explore whether this relationship between $N_s$ and $d_N$ as we have established is also sufficient to guarantee low PINN risk and extending these findings to systems with vector-valued solutions (e.g., the Navier--Stokes equations). Establishing these rigorously could be an exciting direction of future work, particularly for deeper networks with multilayer training. Lastly we note that alleviating the bounded network assumption to have proofs that encompass unbounded networks would also be an important next step.

\section*{Funding}
Sebastien Andre-Sloan and Anirbit Mukherjee acknowledge funding from the UKRI AI Centre for Doctoral Training in Decision Making for Complex Systems,  supported by the Engineering and Physical Sciences Research Council
[EP/Y030826/1].

\bibliography{tmlr}

\clearpage 
\appendix

\section{Proofs of Intermediate Lemmas}
\label{sec:lemmas_pfs}

We now prove the intermediate lemmas.

\begin{proof}[Proof of \cref{lastLemma}]
We define the following empirical risk functions:
\begin{align*}\setlength\abovedisplayskip{6pt}\setlength\belowdisplayskip{6pt}
    \hat{R}_P(p_\vw)&=\frac{1}{N_r}\sum_{i=1}^{N_r}\left(\partial_tp_\vw(\vx_{ri}, t_{ti})-\nabla^2p_\vw(\vx_{ri}, t_{ti})+\norm{\nabla p_\vw(\vx_{ri}, t_{ti})}_2^2\right)^2 \\
		&\quad+ \frac{\lambda_0}{N_0}\sum_{j=1}^{N_0}\left(p_\vw(\vx_{0j},0)-g(\vx_{0j})\right)^2\\
    \hat{R}_S(p_\vw)&=\frac{1}{N_s}\sum_{k=1}^{N_s}(p_\vw(\vx_{sk},t_{sk})-y_{sk})^2.
\end{align*}
Recalling the expression from \cref{def:hjbloss}, we obtain $\hat{R}(p_\vw)=\hat{R}_P(p_\vw)+\lambda_s\hat{R}_S(p_\vw)$.

Let $\phi_i=\E[y_{si}|(\vx_{si},t_{si})]$, $z_i = y_{si} - \phi_i$, and 
$\sigma^2=\frac{1}{N_s}\sum_{i=1}^{N_s}\E[z_i^2]$.
For a given $p_\vw$, define
$$\setlength\abovedisplayskip{6pt}\setlength\belowdisplayskip{6pt}
Z=\frac{1}{\sqrt{N_s}}(z_1,z_2,...,z_{N_s}), ~G=\frac{1}{\sqrt{N_s}}(\phi_1,\phi_2,...,\phi_{N_s}),
$$
$$\setlength\abovedisplayskip{6pt}\setlength\belowdisplayskip{6pt}
F=\frac{1}{\sqrt{N_s}}(p_\vw(\vx_1,t_1),p_\vw(\vx_2,t_2),...,p_\vw(\vx_{N_s},t_{N_s})).
$$
Thus we have, $\norm{G+Z-F}^2=\hat{R}(p_\vw) - \hat{R}_P(p_\vw)$. Further suppose $\norm{Z}^2\geq\sigma^2-\frac{\eta}{6}$ and $\ip{Z}{G}\geq-\frac{\eta}{6}$. Then
\begin{align*}
\norm{Z+G-F}^2&=\norm{Z}^2+2\ip{Z}{G-F} + \norm{G-F}^2=\norm{Z}^2+2\ip{Z}{G}-2\ip{Z}{F}+\norm{G-F}^2\\
&\geq\sigma^2-\frac{\eta}{2}-2\ip{Z}{F}.
\end{align*}
The above assumptions imply that
\begin{align}\setlength\abovedisplayskip{6pt}\setlength\belowdisplayskip{6pt}
    \nonumber \hat{R}(p_\vw)=\lambda_s\norm{Z+G+F}^2 + \hat{R}_P(p_\vw)\geq\lambda_s(\sigma^2-\frac{\eta}{2}-2\ip{Z}{F}) + \hat{R}_P(p_\vw).
\end{align}
Define the following events:
\begin{align}\label{eq:E1-3def}\setlength\abovedisplayskip{6pt}\setlength\belowdisplayskip{6pt}
    \nonumber \mE_1 = \left\{\norm{Z}^2\geq\sigma^2-\frac{\eta}{6}\right\}&,\mE_2 = \left\{\ip{Z}{G}\geq-\frac{\eta}{6}\right\},\mE_3 = \left\{\exists ~p_\vw\in\gH|\hat{R}(p_\vw)\leq\sigma^2-\mu\eta\right\}\\
    \mE_4=&\left\{\exists ~p_\vw\in\gH|\ip{F}{Z}\geq\frac{1}{2\lambda_s}\left(\mu\eta - \sigma^2\right) + \left(\frac{1}{2}\sigma^2 - \frac{1}{4}\eta\right)\right\}.
\end{align}
If $\mE_1,\mE_2,$ and $\mE_3$ hold, then $\mE_4$ holds. Hence, $\Pr(\mE_1 \cap \mE_2 \cap \mE_3)\leq\Pr(\mE_4) $ implies that $  \Pr(\mE_4)\geq 1-\Pr((\mE_1 \cap \mE_2 \cap \mE_3)^c)$. Since $\Pr(\mE_1^c \cup \mE_2^c \cup \mE_3^c)\leq\Pr(\mE_1^c)+\Pr(\mE_2^c)+\Pr(\mE_3^c)\leq 3 - (\Pr(\mE_1)+\Pr(\mE_2)+\Pr(\mE_3))$, this can be combined with the previous bound to get $\Pr(\mE_4)\geq-2+(\Pr(\mE_1)+\Pr(\mE_2)+\Pr(\mE_3))$ and hence
\begin{align}\label{eq:prob}\setlength\abovedisplayskip{6pt}\setlength\belowdisplayskip{6pt}
\Pr(\mE_3)\leq 2-\Pr(\mE_1)-\Pr(\mE_2)+\Pr(\mE_4).
\end{align}
From the assumptions on the labels given in \cref{lossfunction}, the bound of the labels is $|y_{si}|<M$, which implies that $\abs{z_i}=\abs{y_{si}-\E[y_{si}|(\vx_{si},t_{si})]}\leq2M$. This also means $\E[z_{i}|(\vx_{si},t_{si})]=0$ and that $0\leq z_i^2\leq4M^2$.

By Hoeffding's inequality \cite{hoeffding1963probability} \footnote{\begin{theorem}[Hoeffding's Inequality]\label{lem:hoeff}    Let $Z_1$,\ldots,$Z_n$ be independent bounded random variables with $Z_i\in[a,b]$ for all $i$, where $-\infty<a\leq b<\infty$. Then, for all $t\geq0$,
$$\setlength\abovedisplayskip{6pt}\setlength\belowdisplayskip{6pt}
\Pr(\frac{1}{n}\sum_{i=1}^n(Z_i-\E[Z_i])\geq t)\leq\exp{-\frac{2nt^2}{(b-a)^2}},\quad \Pr(\frac{1}{n}\sum_{i=1}^n(Z_i-\E[Z_i])\leq -t)\leq\exp{-\frac{2nt^2}{(b-a)^2}}.
$$
\end{theorem}} (with $a=0,b=4M^2,n=N_s,t=\eta/6,Z_i=z_i^2$ in \cref{lem:hoeff}),
\begin{align}\label{eq:E1bound}\setlength\abovedisplayskip{6pt}\setlength\belowdisplayskip{6pt}
    \Pr\left(\frac{1}{N_s}\sum_{i=1}^{N_s}z_i^2\leq\sigma^2-\frac{\eta}{6}\right)\leq\exp\left(-\frac{N_s\eta^2}{288M^4}\right) \text{ hence } 1-\Pr(\mE_1)\leq\exp\left(-\frac{N_s\eta^2}{288M^4}\right).
\end{align}
In the above, we have used the definition of $\mE_1$ from \cref{eq:E1-3def}. The expectation of $z_i\E[y_{si}|(\vx_{si},t_{si})]=\E[z_i]$ since $\E[y_{si}|(\vx_{si},t_{si})]$ is constant with respect to $(\vx_{si},t_{si})$, so $\E[z_i\mE[y_{si}|(x_{si},t_{si})]=0$. The bound on $z_i\E[y_{si}|(\vx_{si},t_{si})]$ would use the bound of $\abs{z_i}\leq2M$ and $\abs{\E[y_{si}|(\vx_{si},t_{si})]}\leq M$.
Therefore Hoeffding's inequality (with $a=-2M^2,b=2M^2,n=N_s,t=\eta/6,Z_i=z_i\E[y_{si}|(\vx_{si},t_{si})]$ in \cref{lem:hoeff}) can be used again to show that

\begin{align}\label{eq:E2bound}\setlength\abovedisplayskip{6pt}\setlength\belowdisplayskip{6pt}
    \Pr\left(\frac{1}{N_s}\sum_{i=1}^{N_s}z_i\E[y_{si}|(\vx_{si},t_{si})]\leq-\frac{\eta}{6}\right)\leq\exp\left(-\frac{N_s\eta^2}{288M^4}\right)\text{ hence }1-\Pr(\mE_2)\leq\exp\left(-\frac{N_s\eta^2}{288M^4}\right).
\end{align}
Here, we used the definition of $\mE_2$ from \cref{eq:E1-3def} above. Substituting \cref{eq:E1bound,eq:E2bound} into \cref{eq:prob}, we see that
\begin{align}\setlength\abovedisplayskip{6pt}\setlength\belowdisplayskip{6pt}
    \nonumber \Pr(\mE_3)\leq2\exp\left(-\frac{N_s\eta^2}{288M^4}\right)+\Pr(\mE_4).
\end{align}
We can rewrite this using the definitions of $\mE_3$ and $\mE_4$ from \cref{eq:E1-3def} to obtain,
\begin{align}\setlength\abovedisplayskip{6pt}\setlength\belowdisplayskip{6pt}
    \nonumber &\Pr\left(\exists ~p_\vw\in\gH|\hat{R}(p_\vw)\leq\sigma^2-\mu\eta\right)\\
    &\leq2\exp\left(-\frac{N_s\eta^2}{288M^4}\right)+\Pr\left(\exists ~p_\vw\in\gH|\frac{1}{N_s}\sum_{i=1}^{N_s}p_\vw(\vx_{si},t_{si})z_i\geq\frac{1}{2\lambda_s}\left(\mu\eta - \sigma^2\right) + \left(\frac{1}{2}\sigma^2 - \frac{1}{4}\eta\right)\right),
\end{align}
which finishes the proof.
\end{proof}

\begin{proof}[Proof of \cref{secondLemma}]
Recalling the definition of $z_i$ from \cref{thm:main}, let $S=\frac{1}{2\lambda_s}\left(\mu\eta - \sigma^2\right) + \left(\frac{1}{2}\sigma^2 - \frac{1}{4}\eta\right)$ and let the events $\mE_5$ and $\mE_6$ be
\begin{align}\label{def:E5E6}\setlength\abovedisplayskip{6pt}\setlength\belowdisplayskip{6pt}
    \mE_5=\left\{\abs{\frac{1}{N_s}\sum_{i=1}^{N_s}z_i}\geq\frac{S}{2C_1}\right\}, \quad\mE_6=\left\{\exists h_{\vw_{\frac{\eta}{2}}}\in\gH_\eta|\frac{1}{N_s}\sum_{i=1}^{N_s}\E[h_{\vw_{\frac{\eta}{2}}}]z_i>\frac{S}{2}\right\}.
\end{align}
Since $\frac{1}{C_1}|\E[h_{\vw_{\frac{\eta}{2}}}]|\in[0,1]$, if $\mE_5^c$ happens then for such a sample of $\left\{z_i,i=1,...,n\right\}$,
$$\setlength\abovedisplayskip{6pt}\setlength\belowdisplayskip{6pt}
\frac{S}{2C_1}>\abs{\frac{1}{N_s}\sum_{i=1}^{N_s}z_i}\geq\frac{1}{C_1}\abs{\E[h_{\vw_{\frac{\eta}{2}}}]}\abs{\frac{1}{N_s}\sum_{i=1}^{N_s}z_i}\geq\frac{1}{N_sC_1}\sum_{i=1}^{N_s}\E[h_{\vw_{\frac{\eta}{2}}}]z_i\quad\forall h_{\vw_{\frac{\eta}{2}}}\in\gH_\eta.
$$
Hence,  $\mE_5^c$ implies $\mE_6^c$ and 
\begin{align}\label{eq:E6<E5}\setlength\abovedisplayskip{6pt}\setlength\belowdisplayskip{6pt}
\Pr(\mE_6)\leq\Pr(\mE_5).
\end{align}
We recall the calculation done in \cref{lastLemma} that $|z_i|\leq2M$ and that $\E[z_i|(\vx_{si},t_{si})]=0$.
By Hoeffding's inequality (setting $a=-2M,b=2M,n=N_s,t=S/2C_1,Z_i=z_i$ in \cref{lem:hoeff}), 
\begin{align}\label{eq:E5bound}\setlength\abovedisplayskip{6pt}\setlength\belowdisplayskip{6pt}
    \Pr(\mE_5)=\Pr \left (\abs{\frac{1}{N_s}\sum_{i=1}^{N_s}z_i}\geq\frac{S}{2C_1} \right )\leq2\exp{\frac{-N_sS^2}{32M^2C_1^2}}.
\end{align}
Further, Define the following three events:
\begin{align*}\setlength\abovedisplayskip{6pt}\setlength\belowdisplayskip{6pt}
\mE_7&=\left\{\forall h_{\vw_{\frac{\eta}{2}}}\in\gH_\eta|\frac{1}{N_s}\sum_{i=1}^{N_s}\left(h_{\vw_{\frac{\eta}{2}}}(\vx_{si},t_{si})-\E[h_{\vw_{\frac{\eta}{2}}}]\right)z_i\leq\frac{S}{2}\right\},\\
\mE_8&=\left\{\forall h_{\vw_{\frac{\eta}{2}}}\in\gH_\eta|\frac{1}{N_s}\sum_{i=1}^{N_s}\E[h_{\vw_{\frac{\eta}{2}}}]z_i\leq\frac{S}{2}\right\},\\
\mE_9&=\left\{\forall h_{\vw_{\frac{\eta}{2}}}\in\gH_\eta|\frac{1}{N_s}\sum_{i=1}^{N_s}h_{\vw_{\frac{\eta}{2}}}(\vx_{si},t_{si})z_i\leq S\right\}.
\end{align*}
Since $\mE_9$ holds if $\mE_7$ and $\mE_8$ hold so that $\Pr(\mE_7\cap\mE_8)\leq\Pr(\mE_9)$, we have
\begin{align}\label{baseineq}\setlength\abovedisplayskip{6pt}\setlength\belowdisplayskip{6pt}
\Pr(\mE_9^c)\leq\Pr(\mE_7^c)+\Pr(\mE_8^c).
\end{align}
Considering $\Pr(\mE_8^c)$ and using the definition of $\mE_6$ from \cref{def:E5E6}, and the inequalities from \cref{eq:E6<E5,eq:E5bound}, we see that
\begin{align}\label{e8}\setlength\abovedisplayskip{6pt}\setlength\belowdisplayskip{6pt}
    \Pr(\mE_8^c)&=\Pr\left(\exists h_{\vw_{\frac{\eta}{2}}}\in\gH_\eta|\frac{1}{N_s}\sum_{i=1}^{N_s}\E[h_{\vw_{\frac{\eta}{2}}}]z_i>\frac{S}{2}\right)=\Pr(\mE_6)\leq\Pr(\mE_5)\leq2\exp{\frac{-N_sS^2}{32M^2C_1^2}}.
\end{align}

Upper-bounding $\Pr(\mE_7^c)$ shows that
\begin{align}\label{e7}\setlength\abovedisplayskip{6pt}\setlength\belowdisplayskip{6pt}
    \nonumber \Pr(\mE_7^c)&=\Pr\left(\exists h_{\vw_{\frac{\eta}{2}}}\in\gH_\eta|\frac{1}{N_s}\sum_{i=1}^{N_s}\left(h_{\vw_{\frac{\eta}{2}}}(\vx_{si},t_{si})-\E[h_{\vw_{\frac{\eta}{2}}}]\right)z_i>\frac{S}{2}\right)\\
    \nonumber &\leq\Pr\left(\bigcup_{h_{\vw_{\frac{\eta}{2}}}\in\gH_\eta}\abs{\frac{1}{N_s}\sum_{i=1}^{N_s}\left(h_{\vw_{\frac{\eta}{2}}}(\vx_{si},t_{si})-\E[h_{\vw_{\frac{\eta}{2}}}]\right)z_i}\geq\frac{S}{2}\right)\\
    &\leq\sum_{h_{\vw_{\frac{\eta}{2}}}\in\gH_\eta}\Pr\left(\abs{\frac{1}{N_s}\sum_{i=1}^{N_s}\left(h_{\vw_{\frac{\eta}{2}}}(\vx_{si},t_{si})-\E[h_{\vw_{\frac{\eta}{2}}}]\right)z_i}\geq\frac{S}{2}\right).
\end{align}
Again, using Hoeffding’s inequality (with $a=-4MC_1,b=4MC_1,n=N_s,t=S/2,Z_i=\left(h_{\vw_{\frac{\eta}{2}}}(\vx_{si},t_{si})-\E[h_{\vw_{\frac{\eta}{2}}}]\right)z_i$ in \cref{lem:hoeff}),
$$\setlength\abovedisplayskip{6pt}\setlength\belowdisplayskip{6pt}
\Pr\left(\frac{1}{N_s}\sum_{i=1}^{N_s}\abs{\left(h_{\vw_{\frac{\eta}{2}}}(\vx_{si},t_{si})-\E[h_{\vw_{\frac{\eta}{2}}}]\right)z_i}\geq\frac{S}{2}\right)\leq2\exp{\frac{-2N_sS^2}{4\cdot(8MC_1)^2}}.
$$
The number of elements in $\gH_\eta$ can be upper-bounded as $N(\eta,W)\leq\left(\frac{2W\sqrt{d_N}}{\eta}\right)^{d_N}$ \cite[Example 27.1]{Shalev-Shwartz_Ben-David_2014},
 where $N(\eta,W)$ is the number of elements in $\gH_\eta$ given $\eta$ and $W$. This results in the bound
\begin{align}\label{eq:E7bound}\setlength\abovedisplayskip{6pt}\setlength\belowdisplayskip{6pt}
\Pr(E_7^c)\leq\left(\frac{2W\sqrt{d_N}}{\eta}\right)^{d_N}2\exp{\frac{-2N_sS^2}{4\cdot(8MC_1)^2}}.
\end{align}
Substituting the bounds from \cref{e8,eq:E7bound} into \cref{baseineq}, we obtain
\begin{align*}\setlength\abovedisplayskip{6pt}\setlength\belowdisplayskip{6pt}
&\Pr\left(\exists ~h_{\vw_{\frac{\eta}{2}}}\in\gH_\eta|\frac{1}{N_s}\sum_{i=1}^{N_s}h_\vw(\vx_{si},t_{si})z_i\geq S\right)\\
&\quad\leq2\left(\frac{2W\sqrt{d_N}}{\eta}\right)^{d_N}\exp{\frac{-N_sS^2}{2\cdot(8MC_1)^2}} + 2\exp{\frac{-N_sS^2}{32M^2C_1^2}},
\end{align*}
which finishes the proof.
\end{proof}

\begin{proof}[Proof of \cref{lem:perturb}]
The proof proceeds by separately computing the variation in each term of the loss in \cref{def:hjbloss} under perturbation of the weights. 

\paragraph{Initial Condition Loss:} 
The variation of the initial condition loss term at any specific collocation point $\vx_{0j}$ in \cref{def:hjbloss} is given by:
\begin{align*}\setlength\abovedisplayskip{6pt}\setlength\belowdisplayskip{6pt}
&\left(h_{\vw_\frac{\eta}{2}}(\vx_{0j}, 0)-g(\vx_{0j})\right)^2-\left(h_{\vw}(\vx_{0j}, 0)-g(\vx_{0j})\right)^2\\
&=\left(h_{\vw_\frac{\eta}{2}}(\vx_{0j}, 0)-h_{\vw}(\vx_{0j}, 0)\right)\left( h_{\vw_\frac{\eta}{2}}(\vx_{0j}, 0)+h_{\vw}(\vx_{0j}, 0) - 2g(\vx_{0j})\right).
\end{align*}
Invoking the conditions in \cref{def:nets}, we get $h_{\vw_\frac{\eta}{2}}(\vx,t)= f(\mW_{1_{\frac{\eta}{2}}} \vx + \vw_{2_{\frac{\eta}{2}}} t)$, where $W_{1_{\frac{\eta}{2}}}$ and $\vw_{2_{\frac{\eta}{2}}}$ are the $\frac{\eta}{2}-$perturbed weights. Hence,
\begin{align*}\setlength\abovedisplayskip{6pt}\setlength\belowdisplayskip{6pt}
h_{\vw_\frac{\eta}{2}}(\vx_{0j}, 0) - h_{\vw}(\vx_{0j}, 0)&=f(\mW_{1_{\frac{\eta}{2}}} \vx_{0j}) - f(\mW_1 \vx_{0j}) \leq L_1\norm{(\mW_{1_{\frac{\eta}{2}}}-\mW_1)\vx_{0j}}\\
   &\leq L_1\norm{(\mW_{1_{\frac{\eta}{2}}}-\mW_1)}_F\norm{\vx_{0j}}\leq L_1\frac{\eta}{2}\sqrt{n}B.
\end{align*}
Recalling the bound on the predictor, we obtain
$$\setlength\abovedisplayskip{6pt}\setlength\belowdisplayskip{6pt}
    \nonumber h_{\vw_\frac{\eta}{2}}(\vx_{0j}, 0)+h_{\vw}(\vx_{0j}, 0)=f(\mW_{1_{\frac{\eta}{2}}} \vx_{0j}) + f(\mW_1 \vx_{0j}) \leq2C_1.
$$
Thus, the total variation in the initial condition loss is
$$\setlength\abovedisplayskip{6pt}\setlength\belowdisplayskip{6pt}
    \left(h_{\vw_\frac{\eta}{2}}(\vx_{0j}, 0)-g(\vx_{0j})\right)^2-\left(h_{\vw}(\vx_{0j}, 0)-g(\vx_{0j})\right)^2\leq L_1\eta\sqrt{n}B\left(C_1+G\right).
$$

\textbf{Supervised Loss:}
The perturbation bound on the supervised loss follows similarly to the above bound, recalling from \cref{lossfunction} that the labels are bounded in magnitude by $M$, and noting that $\norm{\vx} + \abs{t} \leq \sqrt{n}B + T$:
$$\setlength\abovedisplayskip{6pt}\setlength\belowdisplayskip{6pt}
    \left(h_{\vw_\frac{\eta}{2}}(\vx_{si}, t_{si})-y_{si}\right)^2-\left(h_{\vw}(\vx_{si}, t_{si})-y_{si}\right)^2\leq L_1\eta(\sqrt{n}B+T)\left(C_1+M\right).
$$

\paragraph{PDE Residual Loss:} From \cref{def:nets}, we recall that $\nabla f \in \R^k$ and obtain:
\begin{align*}\setlength\abovedisplayskip{6pt}\setlength\belowdisplayskip{6pt}  
&\partial_t h_{\vw}(\vx,t) =\partial_t f(\mW_1\vx + \vw_2t) =\ip{\vw_2}{\nabla f(\mW_1\vx + \vw_2t)}\\
&\nabla h_{\vw}(\vx,t) =\nabla f(\mW_1\vx + \vw_2t) =\mW_1^\top\nabla f(\mW_1\vx + \vw_2t)\\
&\nabla^2 h_{\vw}(\vx,t) =\nabla^2 f(\mW_1\vx + \vw_2t) =\sum_{i=1}^n\partial_{x_i}[\mW_1^\top\nabla f(\mW_1\vx + \vw_2t)]_i\\
&=\sum_{i=1}^n\partial_{x_i}(\sum_{p=1}^kW_{1,pi}[\nabla f(\mW_1\vx+\vw_2t)]_p)=\sum_{i=1}^n\sum_{p=1}^kW_{1,pi}\partial_{x_i}((\partial_p f)(\mW_1\vx+\vw_2t))\\ 
&=\sum_{i=1}^n\sum_{p=1}^kW_{1,pi}\ip{\mW_{1,col(i)}}{(\grad\partial_pf)(\mW_1\vx+\vw_2t)}.
\end{align*}
In the final term above, $\mW_{1,col(i)}$ is the $i^{th}$column of $\mW_1$ and $\partial_p f : \R^k \rightarrow \R$ is the partial derivative of $f$ with respect to its $p^{th}$ input coordinate. Thus, we are led to the following bounds on the variation of the HJB residual loss under the considered weight perturbation:
\begin{align*}\setlength\abovedisplayskip{6pt}\setlength\belowdisplayskip{6pt}
    &\left(\partial_t h_{\vw}(\vx,t) - \nabla^2 h_{\vw}(\vx,t) + \norm{\nabla h_{\vw}(\vx,t)}^2\right)^2 - \left(\partial_t h_{\vw_\frac{\eta}{2}}(\vx,t) - \nabla^2 h_{\vw_\frac{\eta}{2}}(\vx,t) + \norm{\nabla h_{\vw_\frac{\eta}{2}}(\vx,t)}^2\right)^2\\
    =& \left((\partial_t h_{\vw}(\vx,t) + \partial_t h_{\vw_\frac{\eta}{2}}(\vx,t)) - (\nabla^2 h_{\vw}(\vx,t) + \nabla^2 h_{\vw_\frac{\eta}{2}}(\vx,t)) + \left(\norm{\nabla h_{\vw}(\vx,t)}^2 + \norm{\nabla h_{\vw_\frac{\eta}{2}}(\vx,t)}^2\right)\right)\\
    &\times\left((\partial_t h_{\vw}(\vx,t) - \partial_t h_{\vw_\frac{\eta}{2}}(\vx,t)) + (\nabla^2 h_{\vw_\frac{\eta}{2}}(\vx,t) - \nabla^2 h_{\vw}(\vx,t)) + \left(\norm{\nabla h_{\vw}(\vx,t)}^2 - \norm{\nabla h_{\vw_\frac{\eta}{2}}(\vx,t)}^2\right)\right)\\
    \leq& \left(\norm{\partial_t h_{\vw}(\vx,t) + \partial_t h_{\vw_\frac{\eta}{2}}(\vx,t)} + \norm{\nabla^2 h_{\vw}(\vx,t) + \nabla^2 h_{\vw_\frac{\eta}{2}}(\vx,t)} + \left \vert\norm{\nabla h_{\vw}(\vx,t)}^2 + \norm{\nabla h_{\vw_\frac{\eta}{2}}(\vx,t)}^2\right \vert \right )\\
    &\times\left(\norm{\partial_t h_{\vw}(\vx,t) - \partial_t h_{\vw_\frac{\eta}{2}}(\vx,t)} + \norm{\nabla^2 h_{\vw_\frac{\eta}{2}}(\vx,t) - \nabla^2 h_{\vw}(\vx,t)} + \left \vert \norm{\nabla h_{\vw}(\vx,t)}^2 - \norm{\nabla h_{\vw_\frac{\eta}{2}}(\vx,t)}^2 \right \vert \right).
\end{align*}
Substituting in the above the bounds from \cref{p1,p2,p3,p4,p5,p6} we get
\begin{align*}\setlength\abovedisplayskip{6pt}\setlength\belowdisplayskip{6pt}
&\left(\norm{\partial_t h_{\vw}(\vx,t) + \partial_t h_{\vw_\frac{\eta}{2}}(\vx,t)} + \norm{\nabla^2 h_{\vw}(\vx,t) + \nabla^2 h_{\vw_\frac{\eta}{2}}(\vx,t)} + \left \vert \norm{\nabla h_{\vw}(\vx,t)}^2 + \norm{\nabla h_{\vw_\frac{\eta}{2}}(\vx,t)}^2 \right \vert \right)\\
&\times\left(\norm{\partial_t h_{\vw}(\vx,t) - \partial_t h_{\vw_\frac{\eta}{2}}(\vx,t)} + \norm{\nabla^2 h_{\vw_\frac{\eta}{2}}(\vx,t) - \nabla^2 h_{\vw}(\vx,t)} + \left \vert \norm{\nabla h_{\vw}(\vx,t)}^2 - \norm{\nabla h_{\vw_\frac{\eta}{2}}(\vx,t)}^2 \right \vert \right)\\
&\leq \biggl(\frac{\eta}{2}\left(C_2 + WL_2\left(\norm{\vx}+\abs{t}\right)\right) + 2WC_2+C_3nk\left(2W^2 + \eta W + \left(\frac{\eta}{2}\right)^2\right)+WC_2\eta(WL_2\left(\norm{\vx}+\abs{t}\right)+C_2)\\
&+ 2(WC_2)^2 + \frac{\eta}{2}(WL_2\frac{\eta}{2}(\norm{\vx}+\abs{t}) + (2W+\frac{\eta}{2})C_2))\biggr)\\
&\times\Bigg(\frac{\eta}{2}(C_2 + WL_2(\norm{\vx}+\abs{t}))+C_3nk\left(\eta W + \left(\frac{\eta}{2}\right)^2\right)\\
&+\frac{\eta}{2}(WL_2(\norm{\vx}+\abs{t})+C_2)\left(\frac{\eta}{2}(WL_2(\norm{\vx}+\abs{t}) + C_2)+2WC_2\right)\Bigg)\\
&=a_4\eta^4+a_3\eta^3+a_2\eta^2+b_1\eta,
\end{align*}
where, using that $\norm{\vx}+\abs{t}\leq\sqrt{n}B+T= V$ and the bounds $L_1\leq C_2$ and $L_2\leq C_3$,
\begin{align*}\setlength\abovedisplayskip{6pt}\setlength\belowdisplayskip{6pt}
    a_4&=\frac{1}{16}\left(C_3nk+WC_3V+C_2\right)\left(C_3nk+\left(WC_3V+C_2\right)^2\right)\\
    a_3&=\frac{1}{4}\left(C_3nk+WC_3V+C_2\right)\left((WC_3V+C_2)\left(\frac{1}{2}+WC_2\right)+C_3nkW\right)\\
    &+\frac{1}{4}\left(C_3nk+(WC_3V+C_2)^2\right)\left(C_2\left(\frac{1}{2}+W\left(C_3V+C_2^2+1\right)\right)+W\left(\frac{1}{2}C_3V+C_3nk\right)\right)\\
    \nonumber a_2&=\frac{1}{4}\left(C_3nk+(WC_3V+C_2)^2\right)\left(WC_2(1+WC_2)+C_3nkW^2\right)\\
    &+\left((WC_3V+C_2)\left(\frac{1}{2}+WC_2\right)+C_3nkW\right)\left(C_2\left(\frac{1}{2}+W\left(C_3V+C_2^2+1\right)\right)+W\left(\frac{1}{2}C_3V+C_3nk\right)\right)\\
b_1&=2\left((WC_3V+C_2)\left(1+2WC_2\right)+C_3nkW\right)\left(WC_2(1+WC_2)+C_3nkW^2\right).
\end{align*}
Combining the results from all three loss components, we arrive at 
$$\setlength\abovedisplayskip{6pt}\setlength\belowdisplayskip{6pt}
\hat{R}(h_{\vw_{\frac{\eta}{2}}})\leq\hat{R}(h_\vw) + \lambda_sC_2\eta V\left(C_1+M\right) + \lambda_0C_2\eta\sqrt{n}B\left(C_1+G\right) + a_4\eta^4+a_3\eta^3+a_2\eta^2+b_1\eta,
$$
which proves the result.
\end{proof}

\section{Deriving the HJB PDE from Optimal Stochastic Control}\label{sec:control}

We follow the presentation in works like \citep{HJBviscosity} and \citep{wang2022}, to review how the HJB equations emerge from optimal stochastic control. In stochastic control theory, the state function $\{X_t\}_{0\leq t\leq T}$ is a stochastic process, where $T$ is the time horizon. The evolution of the state function by
\[\setlength\abovedisplayskip{6pt}\setlength\belowdisplayskip{6pt}\left\{\begin{aligned}
    &dX_s=m(s,X_s)ds + \sigma dW_s\quad s\in[t,T]\\
    &X_t=x,
\end{aligned}\right. \]
where $m:[t,T]\times\R^n\rightarrow\R^n$ is the control function and $\{W_s\}$ is a standard $n$-dimensional Brownian motion. The total cost is defined as
$$
J_{x,t}(m)=\E\left[\int_t^Tr(X_s,m,s)ds+g(X_T)\right],
$$
where $r:\R^n\times\R^n\times[0,T]\rightarrow\R$ measures the cost rate during the process and $g:\R^n\rightarrow\R$ measures the final cost at the terminal state. The goal is then to find the value function of the control problem, $u(x,t)=\min_{m\in\gM}J_{x,t}(m)$, where $\gM$ is the set of possible control functions being taken into consideration. It can be shown that the value function $u$ satisfies the HJB Equation
\[\setlength\abovedisplayskip{6pt}\setlength\belowdisplayskip{6pt}\left\{\begin{aligned}
    &\partial_tu(x,t) + \frac{1}{2}\sigma^2\nabla^2u(x,t)+\min_{m\in\gM}[r(x,m(t,x),t)+\ip{\nabla u}{m_t}]=0\\
    &u(x,T)=g(x).
\end{aligned}\right.\]
Denote $m(x,t)=(m_1(x,t),m_2(x,t),\ldots,m_n(x,t))$ as $\vy\in\R^n$ and set $r(x,y,t)=\frac{1}{4}\norm{y}^2_2$. The third term in the HJB Equation becomes
$$\setlength\abovedisplayskip{6pt}\setlength\belowdisplayskip{6pt}
\min_{\vy\in\R^n} \left ( \frac{1}{4}\norm{\vy}_2^2+\ip{\nabla u}{\vy} \right ) =\sum_{i=1}^n\min_{y_i\in\R} \left ( \frac{1}{4}y_i^2+y_i\partial_{x_i}u] \right ).
$$
It can be shown that $\min_{y_i\in\R}[\frac{1}{4}y_i^2+y_i\partial_{x_i}u]=-(\partial_{x_i}u)^2$ by setting $y_i=2\partial_{x_i}u$, which results in the following class of HJB equations:
\[\setlength\abovedisplayskip{6pt}\setlength\belowdisplayskip{6pt}\left\{\begin{aligned}
    &\partial_tu(x,t) + \frac{1}{2}\sigma^2\nabla^2u(x,t)-\norm{\nabla u}_2^2=0 \text{ in $\R^n\times[0,T]$}\\
    &u(x,T)=g(x).
\end{aligned}\right.\]
Taking the transformation $v(x,t)=u(x,T-t)$ and setting $\sigma=\sqrt{2}$ the equation becomes
\[\setlength\abovedisplayskip{6pt}\setlength\belowdisplayskip{6pt}\left\{\begin{aligned}
    &\partial_tv(x,t) - \nabla^2v(x,t) + \norm{\nabla v }_2^2=0 \text{ in $\R^n\times[0,T]$}\\
    &v(x,0)=g(x).
\end{aligned}\right.\]

\section{Depth 2 Sigmoid Activation NN Bounds}\label{sec:depth2}
Let $f:\R^k\rightarrow\R$ be defined as $f(\vz)=\ip{\va}{\sigma(\vz)}$ for $\va\in\R^k$ where $\sigma$ is the sigmoid function acting component-wise. Using the fact that $\abs{\sigma(x)}\leq1$ for all $x\in\R$, the $C_1$ bound can be derived as follows
\begin{equation}\setlength\abovedisplayskip{6pt}\setlength\belowdisplayskip{6pt}
    \nonumber \norm{f}=\norm{\ip{\va}{\sigma(\vz)}}=\abs{\sum_{i=1}^ka_i\sigma(z_i)}\leq\abs{\sum_{i=1}^ka_i}=\norm{\va}_1=C_1.
\end{equation}
The derivative of sigmoid is
$$\setlength\abovedisplayskip{6pt}\setlength\belowdisplayskip{6pt}
\sigma'(x)=\frac{e^x}{(1+e^x)^2},
$$
which has the property $\abs{\sigma'(x)}\leq\frac{1}{4}$ for all $x\in\R$. The gradient of $f$ is $\nabla f_i=a_i\sigma'(z_i)$ so the $C_2$ bound can be derived as follows
$$\setlength\abovedisplayskip{6pt}\setlength\belowdisplayskip{6pt}
\norm{\nabla f}=\sqrt{\sum_{i=1}^k(\nabla f_i)^2}=\sqrt{\sum_{i=1}^k(a_i\sigma'(z_i))^2}\leq\sqrt{\sum_{i=1}^k\left(\frac{a_i}{4}\right)^2}=\frac{1}{4}\sqrt{\sum_{i=1}^k(a_i^2)}=\frac{1}{4}\norm{\va}=C_2.
$$
The second derivative of sigmoid is,
$$\setlength\abovedisplayskip{6pt}\setlength\belowdisplayskip{6pt}
\sigma''(x)=-\frac{e^x(e^x-1)}{(1+e^x)^3},
$$
which also has the property that $\sigma''(x)\leq\frac{1}{4}$ for all $x\in\R$. The elements of the Hessian of $f$ are given by $Hf_{i,j}=\partial_{z_i}\partial_{z_j}f=\partial_{z_i}a_j\sigma'(z_j)$, which is only non-zero when $i=j$, where it becomes $Hf_{i,i}=a_i\sigma''(z_i)$. The $C_3$ bound is therefore
$$\setlength\abovedisplayskip{6pt}\setlength\belowdisplayskip{6pt}
\norm{Hf}_F=\sqrt{\sum_{i=1}^k\sum_{j=1}^k(\partial_{z_i}\partial_{z_j}f)^2}=\sqrt{\sum_{i=1}^k(a_i\sigma''(z_i))^2}\leq\sqrt{\sum_{i=1}^k\left(\frac{a_i}{6\sqrt{3}}\right)^2}=\frac{1}{6\sqrt{3}}\norm{\va}=C_3.
$$

\section{Intermediate Results on Bounds on Derivatives of the Predictor and Its Perturbations}

In this Appendix in the six lemmas given below, we compute the upper bound on the following $6$ terms, $\partial_t h_\vw - \partial_t h_{\vw_\frac{\eta}{2}}$, $\partial_t h_\vw + \partial_t h_{\vw_\frac{\eta}{2}}$, $\|{\nabla h_\vw}\|^2 - \|{\nabla h_{\vw_\frac{\eta}{2}}}\|^2$, $\|{\nabla h_\vw}\|^2 - \|{\nabla h_{\vw_\frac{\eta}{2}}}\|^2$, $\|{\nabla h_\vw}\|^2 + \|{\nabla h_{\vw_\frac{\eta}{2}}}\|^2$, $\nabla^2 h_{\vw_\frac{\eta}{2}} - \nabla^2 h_\vw$ and $\nabla^2 h_\vw + \nabla^2 h_{\vw_\frac{\eta}{2}}$.

\begin{lemma}[Upper bound on $\partial_t h_\vw - \partial_t h_{\vw_\frac{\eta}{2}}$]\label{p1}
The following bound holds:
$$
\partial_t h_\vw - \partial_t h_{\vw_\frac{\eta}{2}} \leq\frac{\eta}{2}(C_2 + WL_2(\norm{\vx}+\abs{t})).
$$
\end{lemma}

\begin{proof}We first write
\begin{align*}\setlength\abovedisplayskip{6pt}\setlength\belowdisplayskip{6pt}
    &\partial_t h_\vw - \partial_t h_{\vw_\frac{\eta}{2}}\\
    &= \ip{\vw_2}{\nabla f(\mW_1\vx + \vw_2t)} - \ip{\vw_{2_{\frac{\eta}{2}}}}{\nabla f(\mW_{1_{\frac{\eta}{2}}}\vx + \vw_{2_{\frac{\eta}{2}}}t)}\\
    &= \ip{\vw_2}{\nabla f(\mW_1\vx + \vw_2t)} - \ip{\vw_2}{\nabla f(\mW_{1_{\frac{\eta}{2}}}\vx + \vw_{2_{\frac{\eta}{2}}}t)}\\
    &+ \ip{\vw_2}{\nabla f(\mW_{1_{\frac{\eta}{2}}}\vx + \vw_{2_{\frac{\eta}{2}}}t)} -  \ip{\vw_{2_{\frac{\eta}{2}}}}{\nabla f(\mW_{1_{\frac{\eta}{2}}}\vx + \vw_{2_{\frac{\eta}{2}}}t)}\\
    &=\ip{\vw_2}{\nabla f(\mW_1\vx + \vw_2t) - \nabla f(\mW_{1_{\frac{\eta}{2}}}\vx + \vw_{2_{\frac{\eta}{2}}}t)} + \ip{\vw_2 - \vw_{2_{\frac{\eta}{2}}}}{\nabla f(\mW_{1_{\frac{\eta}{2}}}\vx + \vw_{2_{\frac{\eta}{2}}}t)}.
\end{align*}
Because of the bound $\|{\vw-\vw_{\frac{\eta}{2}}}\|\leq\frac{\eta}{2}$ and \cref{def:nets}, $\|{\nabla f(\vz_1)-\nabla f(\vz_2)}\|\leq L_2\|{\vz_1-\vz_2}\|$, we have
\begin{align*}\setlength\abovedisplayskip{6pt}\setlength\belowdisplayskip{6pt}
    \ip{\vw_2 - \vw_{2_{\frac{\eta}{2}}}}{\nabla f(\mW_{1_{\frac{\eta}{2}}}\vx + \vw_{2_{\frac{\eta}{2}}}t)}&\leq\norm{\vw_2 - \vw_{2_{\frac{\eta}{2}}}}\norm{\nabla f(\mW_{1_{\frac{\eta}{2}}}\vx + \vw_{2_{\frac{\eta}{2}}}t)}\\
    &\leq\frac{\eta}{2}\norm{\nabla f(\mW_{1_{\frac{\eta}{2}}}\vx + \vw_{2_{\frac{\eta}{2}}}t)}\leq\frac{\eta}{2}C_2
		\end{align*}
		and
		\begin{align*}\setlength\abovedisplayskip{6pt}\setlength\belowdisplayskip{6pt}
    \ip{\vw_2}{\nabla f(\mW_1\vx + \vw_2t) - \nabla f(\mW_{1_{\frac{\eta}{2}}}\vx + \vw_{2_{\frac{\eta}{2}}}t)} &\leq\norm{\vw_2}\norm{\nabla f(\mW_1\vx + \vw_2t) - \nabla f(\mW_{1_{\frac{\eta}{2}}}\vx + \vw_{2_{\frac{\eta}{2}}}t)}\\
    &\leq W\cdot L_2\norm{(\mW_1-\mW_{1_{\frac{\eta}{2}}})\vx + (\vw_2-\vw_{2_{\frac{\eta}{2}}})t}\\
    &\leq WL_2\frac{\eta}{2}(\norm{\vx}+\abs{t}).
\end{align*}
Combining the two completes the proof.
\end{proof}

\begin{lemma}[Upper bound on $\partial_t h_\vw + \partial_t h_{\vw_\frac{\eta}{2}}$]\label{p2}
The following bound holds:
$$
\partial_t h_\vw + \partial_t h_{\vw_\frac{\eta}{2}} \leq\frac{\eta}{2}(C_2 + WL_2(\norm{\vx}+\abs{t})) + 2WC_2.
$$
\end{lemma}

\begin{proof}We first write
\begin{align*}\setlength\abovedisplayskip{6pt}\setlength\belowdisplayskip{6pt}
    &\partial_t h_\vw + \partial_t h_{\vw_\frac{\eta}{2}}\\
    &= \ip{\vw_2}{\nabla f(\mW_1\vx + \vw_2t)} + \ip{\vw_{2_{\frac{\eta}{2}}}}{\nabla f(\mW_{1_{\frac{\eta}{2}}}\vx + \vw_{2_{\frac{\eta}{2}}}t)}\\
    &= \ip{\vw_2}{\nabla f(\mW_1\vx + \vw_2t)} - \ip{\vw_2}{\nabla f(\mW_{1_{\frac{\eta}{2}}}\vx + \vw_{2_{\frac{\eta}{2}}}t)}\\
    &+ \ip{\vw_2}{\nabla f(\mW_{1_{\frac{\eta}{2}}}\vx + \vw_{2_{\frac{\eta}{2}}}t)} +  \ip{\vw_{2_{\frac{\eta}{2}}}}{\nabla f(\mW_{1_{\frac{\eta}{2}}}\vx + \vw_{2_{\frac{\eta}{2}}}t)}\\
    &=\ip{\vw_2}{\nabla f(\mW_1\vx + \vw_2t) - \nabla f(\mW_{1_{\frac{\eta}{2}}}\vx + \vw_{2_{\frac{\eta}{2}}}t)} + \ip{\vw_2 + \vw_{2_{\frac{\eta}{2}}}}{\nabla f(\mW_{1_{\frac{\eta}{2}}}\vx + \vw_{2_{\frac{\eta}{2}}}t)}.
\end{align*}
Recalling that $\|{\vw-\vw_{\frac{\eta}{2}}}\|\leq\frac{\eta}{2}$ and $\norm{\nabla f(\vz_1)-\nabla f(\vz_2)}\leq L_2\|{\vz_1-\vz_2}\|$, we have 
\begin{align*}\setlength\abovedisplayskip{6pt}\setlength\belowdisplayskip{6pt}
    \ip{\vw_2 + \vw_{2_{\frac{\eta}{2}}}}{\nabla f(\mW_{1_{\frac{\eta}{2}}}\vx + \vw_{2_{\frac{\eta}{2}}}t)}&\leq\norm{\vw_2 + \vw_{2_{\frac{\eta}{2}}}}\norm{\nabla f(\mW_{1_{\frac{\eta}{2}}}\vx + \vw_{2_{\frac{\eta}{2}}}t)}\\
    &\leq\left(\frac{\eta}{2}+2W\right)\norm{\nabla f(\mW_{1_{\frac{\eta}{2}}}\vx + \vw_{2_{\frac{\eta}{2}}}t)}\leq\left(\frac{\eta}{2}+2W\right)C_2
		\end{align*}
		and
		\begin{align*}\setlength\abovedisplayskip{6pt}\setlength\belowdisplayskip{6pt}
    &\ip{\vw_2}{\nabla f(\mW_1\vx + \vw_2t) - \nabla f(\mW_{1_{\frac{\eta}{2}}}\vx + \vw_{2_{\frac{\eta}{2}}}t)} \\
		&\quad\quad\quad\leq\norm{\vw_2}\norm{\nabla f(\mW_1\vx + \vw_2t) - \nabla f(\mW_{1_{\frac{\eta}{2}}}\vx + \vw_{2_{\frac{\eta}{2}}}t)}\\
    &\quad\quad\quad\leq W\cdot L_2\norm{(\mW_1-\mW_{1_{\frac{\eta}{2}}})\vx + (\vw_2-\vw_{2_{\frac{\eta}{2}}})t}\leq WL_2\frac{\eta}{2}(\norm{\vx}+\abs{t})
\end{align*}
Combining the two completes the proof.
\end{proof}

\begin{lemma}[Upperbound on $\|{\nabla h_\vw}\|^2 - \|{\nabla h_{\vw_\frac{\eta}{2}}}\|^2$]\label{p3}
The following bound holds:
$$
\norm{\nabla h_\vw}^2 - \norm{\nabla h_{\vw_\frac{\eta}{2}}}^2\leq\frac{\eta}{2}(WL_2(\norm{\vx}+\abs{t})+C_2)\left(\frac{\eta}{2}(WL_2(\norm{\vx}+\abs{t}) + C_2)+2WC_2\right).
$$
\end{lemma}

\begin{proof}We first write
\begin{align*}\setlength\abovedisplayskip{6pt}\setlength\belowdisplayskip{6pt}
    &\norm{\nabla h_\vw}^2 - \norm{\nabla h_{\vw_\frac{\eta}{2}}}^2\\
    &= \ip{\nabla h_\vw}{\nabla h_\vw} - \ip{\nabla h_{\vw_\frac{\eta}{2}}}{\nabla h_{\vw_\frac{\eta}{2}}}\\
    &= \ip{\nabla h_\vw}{\nabla h_\vw} - \ip{\nabla h_\vw}{\nabla h_{\vw_\frac{\eta}{2}}} + \ip{\nabla h_\vw}{\nabla h_{\vw_\frac{\eta}{2}}} - \ip{\nabla h_{\vw_\frac{\eta}{2}}}{\nabla h_{\vw_\frac{\eta}{2}}}\\
    &= \ip{\nabla h_\vw}{\nabla h_\vw - \nabla h_{\vw_\frac{\eta}{2}}} + \ip{\nabla h_\vw - \nabla h_{\vw_\frac{\eta}{2}}}{\nabla h_{\vw_\frac{\eta}{2}}}\\
    &\leq \norm{\nabla h_\vw + \nabla h_{\vw_\frac{\eta}{2}}}\norm{\nabla h_\vw - \nabla h_{\vw_\frac{\eta}{2}}}
\end{align*}
Looking specifically at $\|{\nabla h_\vw - \nabla h_{\vw_\frac{\eta}{2}}}\|$,
\begin{align*}\setlength\abovedisplayskip{6pt}\setlength\belowdisplayskip{6pt}
    &\norm{\nabla h_\vw - \nabla h_{\vw_\frac{\eta}{2}}}\\
    &= \norm{\mW_1^\top\nabla f(\mW_1\vx + \vw_2t) - \mW_{1,\frac{\eta}{2}}^\top\nabla f(\mW_{1,\frac{\eta}{2}}\vx + \vw_{2,\frac{\eta}{2}}t)}\\
    &= \norm{\mW_1^\top(\nabla f(\mW_1\vx + \vw_2t) - \nabla f(\mW_{1,\frac{\eta}{2}}\vx + \vw_{2,\frac{\eta}{2}}t)) + (\mW_1^\top - \mW_{1,\frac{\eta}{2}}^\top)\nabla f(\mW_{1,\frac{\eta}{2}}\vx + \vw_{2,\frac{\eta}{2}}t)}\\
    &\leq \norm{\mW_1^\top}_F\norm{\nabla f(\mW_1\vx + \vw_2t) - \nabla f(\mW_{1,\frac{\eta}{2}}\vx + \vw_{2,\frac{\eta}{2}}t))} + \norm{\mW_1^\top - \mW_{1,\frac{\eta}{2}}^\top}F\norm{\nabla f(\mW_{1,\frac{\eta}{2}}\vx + \vw_{2,\frac{\eta}{2}}t)}\\
    &\leq WL_2\frac{\eta}{2}(\norm{\vx}+\abs{t}) + \frac{\eta}{2}C_2 = \frac{\eta}{2}(WL_2(\norm{\vx}+\abs{t})+C_2).
\end{align*}
Then looking at $\|{\nabla h_\vw + \nabla h_{\vw_\frac{\eta}{2}}}\|$,
\begin{align*}\setlength\abovedisplayskip{6pt}\setlength\belowdisplayskip{6pt}
    &\norm{\nabla h_\vw + \nabla h_{\vw_\frac{\eta}{2}}}\\
    &= \norm{\mW_1^\top\nabla f(\mW_1\vx + \vw_2t) + \mW_{1,\frac{\eta}{2}}^\top\nabla f(\mW_{1,\frac{\eta}{2}}\vx + \vw_{2,\frac{\eta}{2}}t)}\\
    &= \norm{\mW_1^\top(\nabla f(\mW_1\vx + \vw_2t) - \nabla f(\mW_{1,\frac{\eta}{2}}\vx + \vw_{2,\frac{\eta}{2}}t)) + (\mW_1^\top + \mW_{1,\frac{\eta}{2}}^\top)\nabla f(\mW_{1,\frac{\eta}{2}}\vx + \vw_{2,\frac{\eta}{2}}t)}\\
    &\leq \norm{\mW_1^\top}_F\norm{\nabla f(\mW_1\vx + \vw_2t) - \nabla f(\mW_{1,\frac{\eta}{2}}\vx + \vw_{2,\frac{\eta}{2}}t))} + \norm{\mW_1^\top + \mW_{1,\frac{\eta}{2}}^\top}_F\norm{\nabla f(\mW_{1,\frac{\eta}{2}}\vx + \vw_{2,\frac{\eta}{2}}t)}\\
    &\leq WL_2\frac{\eta}{2}(\norm{\vx}+\abs{t}) + (2W+\frac{\eta}{2})C_2.
\end{align*}
Combining these two completes the proof.
\end{proof}

\begin{lemma}[Upper bound on $\|{\nabla h_\vw}\|^2 + \|{\nabla h_{\vw_\frac{\eta}{2}}}\|^2$]\label{p4}
The following bound holds:
$$
\norm{\nabla h_\vw}^2 + \norm{\nabla h_{\vw_\frac{\eta}{2}}}^2\leq WC_2\eta(WL_2(\norm{\vx}+\abs{t})+C_2) + 2(WC_2)^2 + \frac{\eta}{2}(WL_2\frac{\eta}{2}(\norm{\vx}+\abs{t}) + (2W+\frac{\eta}{2})C_2)).
$$
\end{lemma}

\begin{proof}We first write
\begin{align*}\setlength\abovedisplayskip{6pt}\setlength\belowdisplayskip{6pt}
    &\norm{\nabla h_\vw}^2 + \norm{\nabla h_{\vw_\frac{\eta}{2}}}^2\\
    &= \ip{\nabla h_\vw}{\nabla h_\vw} + \ip{\nabla h_{\vw_\frac{\eta}{2}}}{\nabla h_{\vw_\frac{\eta}{2}}}\\
    &= \ip{\nabla h_\vw}{\nabla h_\vw} - \ip{\nabla h_\vw}{\nabla h_{\vw_\frac{\eta}{2}}} + \ip{\nabla h_\vw}{\nabla h_{\vw_\frac{\eta}{2}}} + \ip{\nabla h_{\vw_\frac{\eta}{2}}}{\nabla h_{\vw_\frac{\eta}{2}}}\\
    &= \ip{\nabla h_\vw}{\nabla h_\vw - \nabla h_{\vw_\frac{\eta}{2}}} + \ip{\nabla h_\vw + \nabla h_{\vw_\frac{\eta}{2}}}{\nabla h_{\vw_\frac{\eta}{2}}}\\
    &\leq\norm{\nabla h_\vw}\norm{\nabla h_\vw - \nabla h_{\vw_\frac{\eta}{2}}} + \norm{\nabla h_\vw + \nabla h_{\vw_\frac{\eta}{2}}}\norm{\nabla h_{\vw_\frac{\eta}{2}}}.
\end{align*}
Using the bounds on $\|{\nabla h_\vw - \nabla h_{\vw_\frac{\eta}{2}}}\|$ and $\|{\nabla h_\vw + \nabla h_{\vw_\frac{\eta}{2}}}\|$ derived above, we can bound this by
$$
\norm{\nabla h_\vw}(\frac{\eta}{2}(WL_2(\norm{\vx}+\abs{t})+C_2)) + (WL_2\frac{\eta}{2}(\norm{\vx}+\abs{t}) + (2W+\frac{\eta}{2})C_2))\norm{\nabla h_{\vw_\frac{\eta}{2}}}.
$$
The following two bounds hold:
\begin{align*}\setlength\abovedisplayskip{6pt}\setlength\belowdisplayskip{6pt}
\norm{\nabla h_\vw} &= \norm{\mW_1^\top\nabla f(\mW_1\vx + \vw_2t)}\leq\norm{\mW_1^\top}_F\norm{\nabla f(\mW_1\vx + \vw_2t)}\leq WC_2,\\
\norm{\nabla h_{\vw_\frac{\eta}{2}}} &= \norm{\mW_1^\top\nabla f(\mW_1\vx + \vw_2t)}\leq\norm{\mW_{1,\frac{\eta}{2}}^\top}_F\norm{\nabla f(\mW_{1,\frac{\eta}{2}}\vx + \vw_{2,\frac{\eta}{2}}t)}\leq (W+\frac{\eta}{2})C_2.
\end{align*}
Hence,
\begin{align*}\setlength\abovedisplayskip{6pt}\setlength\belowdisplayskip{6pt}
    &\norm{\nabla h_\vw}^2 + \norm{\nabla h_{\vw_\frac{\eta}{2}}}^2\\
    &\leq WC_2\left(\frac{\eta}{2}(WL_2(\norm{\vx}+\abs{t})+C_2)\right) + \left(W+\frac{\eta}{2}\right)C_2\left(WL_2\frac{\eta}{2}(\norm{\vx}+\abs{t}) + (2W+\frac{\eta}{2})C_2\right)\\
    &=WC_2\eta\left(WL_2(\norm{\vx}+\abs{t})+C_2\right) + 2(WC_2)^2 + \frac{\eta}{2}\left(WL_2\frac{\eta}{2}(\norm{\vx}+\abs{t}) + (2W+\frac{\eta}{2})C_2\right),
\end{align*}
which proves the result.
\end{proof}

\begin{lemma}[Upper bound on $\nabla^2 h_{\vw_\frac{\eta}{2}} - \nabla^2 h_\vw$]\label{p5}
The following bound holds:
$$
\nabla^2 h_{\vw_\frac{\eta}{2}} - \nabla^2 h_\vw\leq C_3nk\left(\eta W + \left(\frac{\eta}{2}\right)^2\right).
$$
\end{lemma}
\begin{proof}We first write
\begin{align*}\setlength\abovedisplayskip{6pt}\setlength\belowdisplayskip{6pt}
    &\nabla^2 h_{\vw_\frac{\eta}{2}} - \nabla^2 h_\vw\\
    &= \sum_{i=1}^n\sum_{p=1}^kW_{1,\frac{\eta}{2},pi}\ip{\mW_{1,\frac{\eta}{2},col(i)}}{(\grad\partial_pf)(\mW_{1,\frac{\eta}{2}}\vx+\vw_{2,\frac{\eta}{2}}t)} - \sum_{i=1}^n\sum_{p=1}^kW_{1,pi}\ip{\mW_{1,col(i)}}{(\grad\partial_pf)(\mW_1\vx+\vw_2t)}\\
    &= \sum_{i=1}^n\sum_{p=1}^k\left(W_{1,\frac{\eta}{2},pi}\ip{\mW_{1,\frac{\eta}{2},col(i)}}{(\grad\partial_pf)(\mW_{1,\frac{\eta}{2}}\vx+\vw_{2,\frac{\eta}{2}}t)} - W_{1,pi}\ip{\mW_{1,col(i)}}{(\grad\partial_pf)(\mW_1\vx+\vw_2t)}\right)\\
    &\leq \sum_{i=1}^n\sum_{p=1}^k\left(W_{1,\frac{\eta}{2},pi}\norm{\mW_{1,\frac{\eta}{2},col(i)}}\norm{(\grad\partial_pf)(\mW_{1,\frac{\eta}{2}}\vx+\vw_{2,\frac{\eta}{2}}t)} - W_{1,pi}\norm{\mW_{1,col(i)}}\norm{(\grad\partial_pf)(\mW_1\vx+\vw_2t)}\right)
\end{align*}
Using the property of $f$ that $\norm{Hf(\vz)}_F\leq C_3$, we can bound this by
\begin{align*}\setlength\abovedisplayskip{6pt}\setlength\belowdisplayskip{6pt}
    &\sum_{i=1}^n\sum_{p=1}^k\left(W_{1,\frac{\eta}{2},pi}\norm{\mW_{1,\frac{\eta}{2},col(i)}}C_3 - W_{1,pi}\norm{\mW_{1,col(i)}}C_3\right)\leq \sum_{i=1}^n\sum_{p=1}^k\left(W_{1,\frac{\eta}{2},pi}(W+\frac{\eta}{2})C_3 - W_{1,pi}WC_3\right)\\
    &= C_3\sum_{i=1}^n\sum_{p=1}^k\left((W_{1,\frac{\eta}{2},pi} - W_{1,pi})W + \frac{\eta}{2}W_{1,\frac{\eta}{2},pi}\right)\leq C_3\sum_{i=1}^n\sum_{p=1}^k\left(\frac{\eta}{2}W + \frac{\eta}{2}(W+\frac{\eta}{2}))\right)\\
    &= C_3\sum_{i=1}^n\sum_{p=1}^k\left(\eta W + \left(\frac{\eta}{2}\right)^2\right)= C_3nk\left(\eta W + \left(\frac{\eta}{2}\right)^2\right).
\end{align*}
The result now follows.
\end{proof}

\begin{lemma}[Upper bound on $\nabla^2 h_{\vw_\frac{\eta}{2}} + \nabla^2 h_\vw$]\label{p6}
The following bound holds:
$$\nabla^2 h_{\vw_\frac{\eta}{2}} + \nabla^2 h_\vw\leq C_3nk\left(2W^2 + \eta W + \left(\frac{\eta}{2}\right)^2\right).
$$
\end{lemma}
\begin{proof}We first write
\begin{align*}\setlength\abovedisplayskip{6pt}\setlength\belowdisplayskip{6pt}
    &\nabla^2 h_{\vw_\frac{\eta}{2}} + \nabla^2 h_\vw\\
    &= \sum_{i=1}^n\sum_{p=1}^kW_{1,\frac{\eta}{2},pi}\ip{\mW_{1,\frac{\eta}{2},col(i)}}{(\grad\partial_pf)(\mW_{1,\frac{\eta}{2}}\vx+\vw_{2,\frac{\eta}{2}}t)} + \sum_{i=1}^n\sum_{p=1}^kW_{1,pi}\ip{\mW_{1,col(i)}}{(\grad\partial_pf)(\mW_1\vx+\vw_2t)}\\
    &= \sum_{i=1}^n\sum_{p=1}^k\left(W_{1,\frac{\eta}{2},pi}\ip{\mW_{1,\frac{\eta}{2},col(i)}}{(\grad\partial_pf)(\mW_{1,\frac{\eta}{2}}\vx+\vw_{2,\frac{\eta}{2}}t)} + W_{1,pi}\ip{\mW_{1,col(i)}}{(\grad\partial_pf)(\mW_1\vx+\vw_2t)}\right)\\
    &\leq \sum_{i=1}^n\sum_{p=1}^k\left(W_{1,\frac{\eta}{2},pi}\norm{\mW_{1,\frac{\eta}{2},col(i)}}\norm{(\grad\partial_pf)(\mW_{1,\frac{\eta}{2}}\vx+\vw_{2,\frac{\eta}{2}}t)} + W_{1,pi}\norm{\mW_{1,col(i)}}\norm{(\grad\partial_pf)(\mW_1\vx+\vw_2t)}\right)
\end{align*}
Using the property of $f$ that $\norm{Hf(\vz)}_F\leq C_3$, we can bound this by
\begin{align*}\setlength\abovedisplayskip{6pt}\setlength\belowdisplayskip{6pt}
    &\sum_{i=1}^n\sum_{p=1}^k\left(W_{1,\frac{\eta}{2},pi}\norm{\mW_{1,\frac{\eta}{2},col(i)}}C_3 + W_{1,pi}\norm{\mW_{1,col(i)}}C_3\right)\leq \sum_{i=1}^n\sum_{p=1}^k\left(W_{1,\frac{\eta}{2},pi}(W+\frac{\eta}{2})C_3 + W_{1,pi}WC_3\right)\\
    &= C_3\sum_{i=1}^n\sum_{p=1}^k\left((W_{1,\frac{\eta}{2},pi} + W_{1,pi})W + \frac{\eta}{2}W_{1,\frac{\eta}{2},pi}\right)\leq C_3\sum_{i=1}^n\sum_{p=1}^k\left((2W + \frac{\eta}{2})W + \frac{\eta}{2}(W+\frac{\eta}{2}))\right)\\
    &= C_3\sum_{i=1}^n\sum_{p=1}^k\left(2W^2 + \eta W + \left(\frac{\eta}{2}\right)^2\right)= C_3nk\left(2W^2 + \eta W + \left(\frac{\eta}{2}\right)^2\right).
\end{align*}
The result now follows.
\end{proof}

\end{document}